\documentclass[a4paper,11pt]{article}

\usepackage{graphicx}
\usepackage{lineno}
\usepackage{amsfonts}
\usepackage{amsmath}
\usepackage{algorithm}
\usepackage{algorithmic}
\usepackage{lipsum}
\usepackage{epstopdf}
\usepackage[colorlinks]{hyperref}
\usepackage{cleveref}

\usepackage[totalwidth=480pt,totalheight=680pt]{geometry}
\newenvironment{keywords}{\footnotesize {\bf  Key words:}}{}
\newtheorem{theorem}{Theorem}
\newtheorem{definition}{Definition}
\newenvironment{proof}{{\noindent\it Proof: }}{\hfill $\square$\par}
\newcommand\email{{\bf email: }}
\ifpdf
  \DeclareGraphicsExtensions{.eps,.pdf,.png,.jpg}
\else
  \DeclareGraphicsExtensions{.eps}
\fi

\usepackage{enumitem}
\setlist[enumerate]{leftmargin=.5in}
\setlist[itemize]{leftmargin=.5in}

\graphicspath{ {./src/} }
\begin{document}
\newcommand{\R}{\mathbb{R}}
\newcommand{\N}{\mathbb{N}}
\newcommand{\Z}{\mathbb{Z}}
\newcommand{\F}{\mathbb{F}}
\newcommand{\C}{\mathbb{C}}
\newcommand{\D}{\mathbb{D}}
\newcommand{\abs}[1]{\left\vert #1 \right\vert}
\newcommand{\norm}[1]{\left\Vert #1 \right\Vert}
\renewcommand{\Re}{\text{Re}}
\renewcommand{\Im}{\text{Im}}
\newcommand{\Part}[2]{\frac{\partial #1}{\partial #2}}
\newcommand{\argmin}{\mathop{\arg\min}}
\newcommand{\argmax}{\mathop{\arg\max}}
\newcommand{\etal}{\textit{et al. }}
% \providecommand{\keywords}[1]{\textbf{\text{Index terms---}}#1}

% \begin{frontmatter}
\title{Shape Prior Segmentation Guided by Harmonic Beltrami Signature}

\author{
    Chenran Lin\thanks{Department of Mathematics, The Chinese University of Hong Kong, Hong Kong (\email{crlin@math.cuhk.edu.hk})}
    \and
    Lok Ming Lui\thanks{Department of Mathematics, The Chinese University of Hong Kong, Hong Kong (\email{lmlui@math.cuhk.edu.hk})}
}
% \headers{Harmonic Beltrami Signature}{Chenran Lin, and Lok Ming Lui}

\maketitle

\begin{abstract}
    This paper presents a novel shape prior segmentation method guided by the Harmonic Beltrami Signature (HBS). 
    The HBS is a shape representation fully capturing 2D simply connected shapes, exhibiting resilience against perturbations and invariance to translation, rotation, and scaling.
    The proposed method integrates the HBS within a quasi-conformal topology preserving segmentation framework, leveraging shape prior knowledge to significantly enhance segmentation performance, especially for low-quality or occluded images. 
    The key innovation lies in the bifurcation of the optimization process into two iterative stages: 
    1) The computation of a quasi-conformal deformation map, which transforms the unit disk into the targeted segmentation area, driven by image data and other regularization terms; 
    2) The subsequent refinement of this map is contingent upon minimizing the $L_2$ distance between its Beltrami coefficient and the reference HBS.
    This shape-constrained refinement ensures that the segmentation adheres to the reference shape(s) by exploiting the inherent invariance, robustness, and discerning shape discriminative capabilities afforded by the HBS.
    Extensive experiments on synthetic and real-world images validate the method's ability to improve segmentation accuracy over baselines, eliminate preprocessing requirements, resist noise corruption, and flexibly acquire and apply shape priors. 
    Overall, the HBS segmentation framework offers an efficient strategy to robustly incorporate the shape prior knowledge, thereby advancing critical low-level vision tasks.
\end{abstract}


\section{Introduction}
\label{section: intro}

Image segmentation stands as a cornerstone within the realm of computer vision, holding a pivotal role across diverse fields. Its essence lies in the precise localization of target objects within an image. Nonetheless, challenges such as poor image quality, occluded objects, blurred boundaries, and other complexities often impede segmentation efforts. In many instances, supplementary information exists regarding the shape of the target object, extending beyond what is directly discernible in the given image, commonly referred to as shape prior information. Integrating such prior knowledge into baseline segmentation algorithms represents a promising avenue for enhancing segmentation performance.

Over the past few decades, a multitude of algorithms for the shape prior segmentation have emerged. These approaches can be roughly classified into three categories based on the type of shape priors employed:

The first class uses geometric properties to describe target objects, such as convexity \cite{luoConvexShapePrior2019,yanConvexityShapePrior2020,gorelickConvexityShapePrior2017}, local convexity \cite{siuImageSegmentationPartial2020}, star shape \cite{veksler2008star}, compactness \cite{das2009semiautomatic}, geodesic star \cite{gulshan2010geodesic}, hedgehog \cite{isack2016hedgehog}, etc. While these methods can improve the performance of baseline algorithms on specific types of shapes, they obviously utilize only a limited portion of the shape information and struggle to differentiate further images sharing similar features.

The second class adopts a more intuitive approach, where researchers directly employ the target object's boundary \cite{bressonVariationalModelObject2006,malcolmGraphCutSegmentation2007,saitoJointOptimizationSegmentation2016,chanLevelSetBased2005,eltanbolyLevelSetsbasedImage2019} or intensity \cite{freedmanInteractiveGraphCut2005,cootesStatisticalModelsAppearance2001} as the shape prior. When dealing with a large number of template images, statistical techniques like PCA have been applied to better summarize the characteristics from the training set \cite{yangParallelizableRobustImage2020,eltanbolyLevelSetsbasedImage2019,saitoJointOptimizationSegmentation2016}. These approaches benefit from the ease of collecting training data and the strong interpretability of resulting shape priors. However, these methods require additional data preprocessing during training and inference, such as alignment, registration, and intensity remapping. The design of preprocessing significantly influences the obtained shape prior and consequently affects the final segmentation results.

The third class utilizes innovative shape representations, calculated based on the given images and capturing their specific features. These shape representations serve as shape priors \cite{yeoSegmentationBiomedicalImages2014}. Undoubtedly, the performance of the segmentation algorithm is significantly influenced by the representation, and thus, selecting a powerful representation is paramount.

The HBS \cite{linHarmonicBeltramiSignature2022} utilized in this paper is an instance of such a potent representation for 2D simply connected shapes. The HBS is uniquely determined by the given shape, invariant under translation, rotation, and scaling, and robust under small perturbations of the shape. The distance between different HBS can be measured easily by the $L_2$ norm because of the concise and unified form, a complex function defined on the unit disk. Notably, the original shape can be faithfully reconstructed from its corresponding HBS, indicating the comprehensive encapsulation of shape information within the signature. With the above advantages, the HBS can be integrated into the quasi-conformal topology preserving segmentation model \cite{chanTopologyPreservingImageSegmentation2018} as a shape prior to guide the segmentation process and gain better performance, which becomes our novel HBS segmentation model.

To delve deeper, the result of the quasi-conformal segmentation model is a Beltrami coefficient, whereas the HBS is just an exceptional Beltrami coefficient. Such an intrinsic connection between the two motivates us to impose specific constraints on the quasi-conformal segmentation model, aligning the resultant Beltrami coefficient with an HBS. 
Concurrently, the boundary of the target object is extracted from the reference image, and then the corresponding HBS is computed as the shape prior. The option to utilize their average HBS is considered in scenarios where multiple images are provided. 
Subsequently, a comparison between the output Beltrami coefficient and the shape prior HBS is conducted by the $L_2$ norm, with iterative minimization to reduce their disparity.
% The HBS equips our proposed model with the following benefits:
The incorporation of the HBS into our proposed model endows it with the following benefits:
\begin{enumerate}
    \item Our model ensures a certain level of similarity between the segmentation result and the reference shapes, leading to improved segmentation performance even for low-quality images;
    \item The invariance of HBS under translation, rotation, and scaling obviates the need for image preprocessing in our proposed model;
    \item The robustness of HBS against minor shape perturbations enhances the stability and reduces the sensitivity to noise in our model;
    \item The property of HBS preserving all shape information provides our model with a strong capability to discern target objects.
\end{enumerate}

The rest of the paper is organized as follows: 
\cref{section: related work} shows some related topics about shape prior segmentation; 
\cref{section: background} introduces some theoretical background; 
\cref{section: main} explains our proposed HBS segmentation model in detail; 
\cref{section: algorithm} gives the numerical implementation; 
\cref{section: exp} reports our experimental results; 
the paper is concluded in \cref{section: conclusion}, and we point out several future directions.

\section{Related works}\label{section: related work}
\subsection{Traditional segmentation}
Active contour or deformable models provide a practical framework for object segmentation and have been widely used in image segmentation.
Kass et al. \cite{kassSnakesActiveContour1988} first introduced the active contour method, evolving a contour towards object boundaries by minimizing an energy function.
Osher et al. \cite{osher1988fronts} and Sethian et al. \cite{adalsteinsson1995fast} developed the level set method, using level sets of a higher dimensional function to enable the implicit representation of curves.
The level set method quickly attracted researchers' attention, and many related works have emerged, such as edge-based models \cite{caselles1993geometric,caselles1995geodesic,yezzi1997geometric}, region-based models \cite{mumford1989optimal,chan2001active}, and shape prior models \cite{leventon2002statistical,rousson2002shape}.
It is worth noting that Chan et al. \cite{chan2001active} enhanced the original level set model into a piecewise constant version, resulting in the well-known CV model, the first model to successfully reduce the heavy dependency on edge information.

Except for the active contour method, several baseline algorithms have been proposed.
Graph cuts \cite{boykov2001interactive,beichel2012liver,boykov2006graph} formulate image segmentation as a minimum cut or maximum flow problem on a graph, where pixels are mapped to graph nodes and edges are weighted based on pairwise pixel similarities.
The watershed transformation \cite{roerdink2000watershed,parvati2008image,cousty2008watershed} segments images by treating intensity as a topographical surface and flooding from regional minima markers.
$K$-means \cite{dhanachandra2015image,pappas1989adaptive} clustering groups pixels into a predefined number of clusters $K$ based on feature similarity to perform segmentation.
Mean shift \cite{comaniciu1999mean,tao2007color} is a nonparametric technique that finds density distribution modes to cluster pixels into homogeneous regions without predefined $K$.
The random walker algorithm \cite{grady2006random,dong2015sub} simulates random walks from manually labeled seed pixels to probabilistically label other pixels for interactive segmentation.

\subsection{Shape prior segmentation}
Utilizing shape priors to guide segmentation is a highly viable approach for enhancing the performance of baseline algorithms, especially in complex scenarios. According to how people use shape priors, these methods can be categorized into three main groups.

Geometric features have caught some researchers' attention, with convexity being one of the most popular properties among this category of shape prior segmentation methods. Gorelick et al. \cite{gorelickConvexityShapePrior2017} utilized 3-clique potentials to penalize 1-0-1 configurations on straight lines and employed an efficient iterative trust region approach for optimization. Luo et al. \cite{luoConvexShapePrior2019} proposed a model incorporating convex multi-object segmentation using the signed distance function corresponding to their union and a Gaussian mixture method. Beyond convexity, properties such as compactness \cite{das2009semiautomatic}, star shape \cite{veksler2008star}, geodesic star \cite{gulshan2010geodesic}, hedgehog \cite{isack2016hedgehog}, and elliptical shape \cite{slabaughGraphCutsSegmentation2005}, can also effectively improve the segmentation accuracy and precision.

Conversely, an image's boundary or intensity information is a valuable form of prior knowledge. Segmentation models under the level set framework can combine the shape prior distance components into their energy functions. Chan et al. 
 \cite{chanLevelSetBased2005} defined the shape distance by the Heaviside function of shapes and then put it into his CV model. Vu et al.  \cite{vuShapePriorSegmentation2008} adopted a similar approach but normalized shapes based on their 3rd-order moments to reduce differences caused by variations in position and angle. Cootes et al.  \cite{cootesStatisticalModelsAppearance2001} treated the histogram of pixel density distribution within the target region as texture information and landmarks of the target shape as boundary information to guide segmentation.

Given plenty of template images, a statistical shape model(SSM) can generate a feature space and search for feasible segmentation within it. Malcolm et al. \cite{malcolmGraphCutSegmentation2007} uses kernel PCA to learn a statistical model of relevant shapes. Nakagomi et al. \cite{nakagomi2013multi} proposed fully automated segmentation using multiple shape priors, which were selected from eigenshapes generated through the uniform sampling of the first two eigenmodes of an SSM. Saito et al. \cite{saitoJointOptimizationSegmentation2016} proposed an algorithm that allows the selection of an optimal shape among the eigenshape space generated from SSM by conducting a branch-and-bound search and then does not require the construction of a hierarchical clustering tree before graph-cut segmentation. Yeo et al. \cite{yeoSegmentationBiomedicalImages2014} incorporated statistical shape information into the Bayesian formulation of the segmentation model using nonparametric shape density distribution. Eltanboly et al. \cite{eltanbolyLevelSetsbasedImage2019} built SSM by modelling the empirical PDF for the intensity level distribution with a linear combination of Gaussians(LCG) and modified an Expectation-Maximization(EM) algorithm to handle the LCGs.

\subsection{Quasi-conformal segmentation}
The quasi-conformal theory has been widely used in the field of computer vision in recent years \cite{zhang2023deformationinvariant,lyu2024spherical,guo2023automatic,lyu2023two,zhu2022parallelizable,zhang2022unifying,zhang2022new,zhang2022nondeterministic,zhang2021quasi},
and many segmentation algorithms based on it have significantly developed. Compared with the other methods, quasi-conformal segmentation is closely related to the proposed model, sharing a similar lineage.
Chan et al. \cite{chanTopologyPreservingImageSegmentation2018} proposed a new approach using the Beltrami representation of a shape for topology-preserving image segmentation. The topology of the segmentation result can be guaranteed by the given simple template by imposing only one constraint on the Beltrami representation, which can be handled easily. Subsequently, Siu et al. \cite{siu2020image} began to focus on partial convexity in images and incorporated a registration-based segmentation model with a specially designed convexity constraint.  Zhang et al. \cite{zhang2021topology,zhangTopologyConvexitypreservingImage2021} further expanded the application domain to 3D images and proposed segmentation models with the hyperelastic regularization and convexity regularization. Zhang et al. \cite{zhang2022new} combined quasi-conformal segmentation with a deep learning framework, designed a Topology-Preserving Segmentation Network(TPSN) and applied it in the field of medical imaging, which produced reliable results even in challenging cases.

\section{Theoretical basis}\label{section: background}
\subsection{Quasi-conformal mapping and Beltrami equation}
A complex function $f: \Omega \subset \C \rightarrow \C$ is said to be \textit{quasi-conformal} associated to $\mu$ if $f$ is orientation-preserving and satisfies the following \textit{Beltrami equation}:
\begin{equation}\label{eq: beltrami eq}
    \Part{f}{\overline{z}} = \mu(z) \Part{f}{z}
\end{equation}
where $\mu(z)$ is a complex-valued Lebesgue measurable function satisfying $\norm{\mu}_\infty < 1$. More specifically, this $\mu: \Omega \rightarrow \D$ is called the \textit{Beltrami coefficient} of $f$
\begin{equation}\label{eq: mu def}
    \mu = \frac{f_{\overline{z}}}{f_{z}}
\end{equation}

In terms of the metric tensor, consider the effect of the pullback under $f$ of the Euclidean metric $ds^2_E$. The resulting metric is given by:
\begin{equation}
    f^*(ds^2_E) = \abs{\Part{f}{z}}^2 \abs{dz + \mu(z)d\overline{z}}^2
\end{equation}
which, relative to the background Euclidean metric $dz$ and $d\overline{z}$, has eigenvalue $(1+\abs{\mu})^2 \abs{\Part{f}{z}}^2$ and $(1-\abs{\mu})^2 \abs{\Part{f}{z}}^2$.

Therefore, inside the local parameter domain around some point $p$, $f$ can be considered as a map composed of a translation to $f(p)$ together with the multiplication of a stretch map $S(z) = z + \mu(p)\overline{z}$ and conformal function $f_z(p)$, which may be expressed as follows:
\begin{equation}\label{eq: local f}
    f(z) =  f(p)+S(z)f_z(p) = f(p)+(z+\mu(p)\overline{z})f_z(p).
\end{equation}
$S(z)$ makes $f$ map a small circle to a small ellipse, and all the conformal distortion of $f$ is caused by $\mu$. To form $\mu(p)$, we can determine the angles of the directions of maximal magnification and shrinkage and the amount of them as well. Specially, the angle of maximal magnification is $\arg(\mu(p))/2$ with magnifying factor $1+\abs{\mu(p)}$; the angle of maximal shrinkage is the orthogonal angle $\arg(\mu(p))/2 - \pi/2$ with shrinkage factor $1-\abs{\mu(p)}$. The distortion or dilation is given by:
\begin{equation}
    K = \frac{1+\abs{\mu(p)}}{1-\abs{\mu(p)}}.
\end{equation}
Thus, the Beltrami coefficient $\mu$ gives us important information about the properties of the map (see \cref{fig: qc}), and $\mu$ is a measure of non-conformality. In particular, the map $f$ is conformal around a small neighborhood of $p$ when $\mu(p)=0$ and if $\mu(z)=0$ everywhere on $\Omega$, $f$ us called \textit{conformal} or \textit{holomorphic} on $\Omega$.

\begin{figure}
    \begin{center}
        \includegraphics[width=7.6cm]{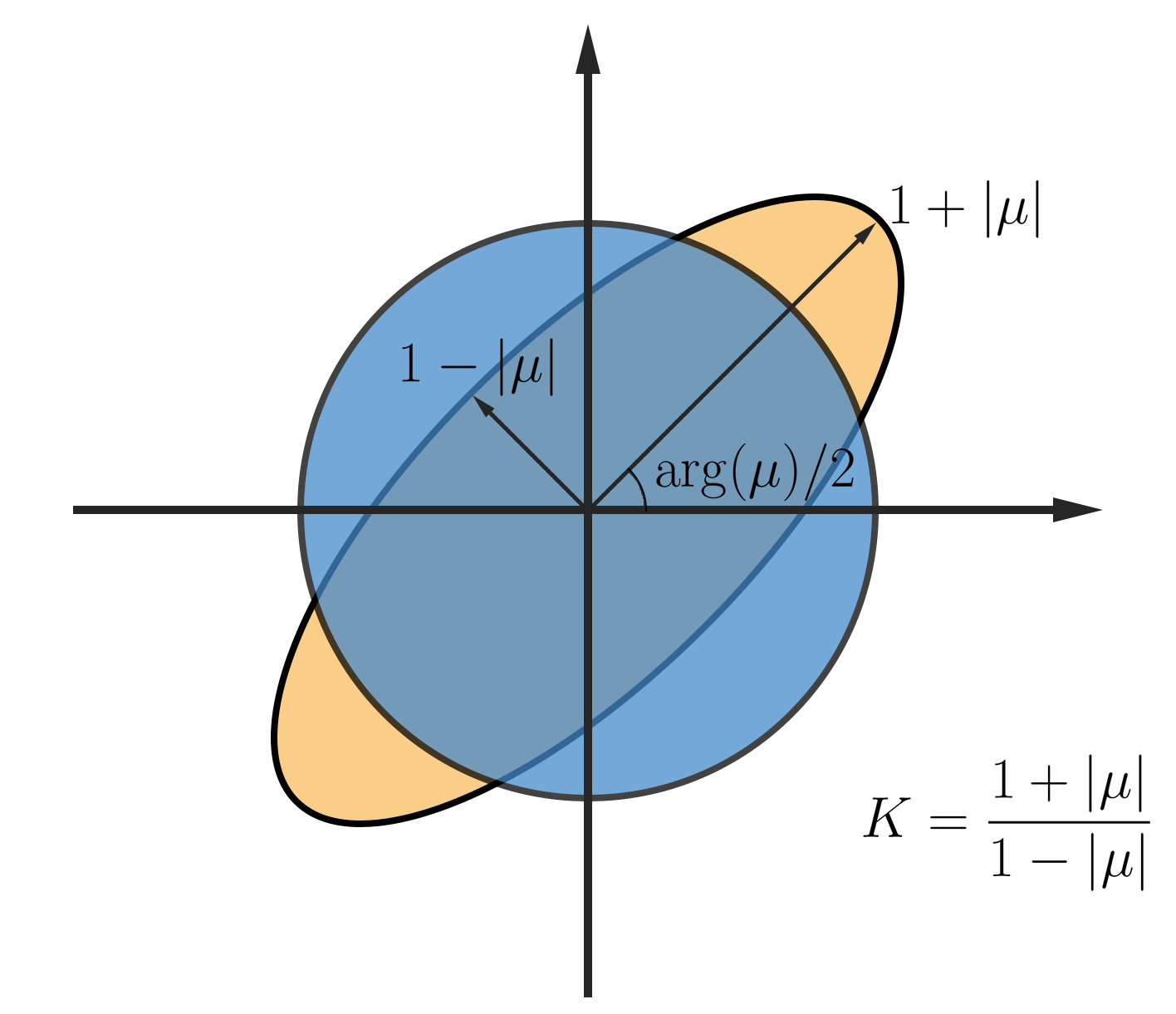}
    \end{center}
    \caption{Quasi-conformal maps infinitesimal circles to ellipses. The Beltrami coefficient measures the distortion or dilation of the ellipse under the QC map.}
    \label{fig: qc}
\end{figure}

Note that there is a one-to-one correspondence between the quasi-conformal mapping $f$ and its Beltrami coefficient $\mu$. Given $f$, there exists a Beltrami coefficient $\mu$ satisfying the Beltrami equation by equation \cref{eq: mu def}. Conversely, the following theorem states that given an admissible Beltrami coefficient $\mu$, a quasi-conformal mapping $f$ always exists associated with this $\mu$.

\begin{theorem}[Measurable Riemannian Mapping Theorem]\label{thm: Measurable Riemannian Mapping Theorem}
    Suppose $\mu: \C \rightarrow \C$ is Lebesgue measurable satisfying $\norm{\mu}_\infty <1$; then, there exists a quasi-conformal homeomorphism $f$ from $\C$ onto itself, which is in the Sobolev space $W_{1,2}(\C)$ and satisfies the Beltrami equation in the distribution sense. The associated quasi-conformal homeomorphism $f$ is unique up to a Mobi\"us transformation. Furthermore, by fixing $0$, $1$ and $\infty$, the $f$ is uniquely determined.
\end{theorem}

\begin{figure}
    \begin{center}
        \includegraphics[width=\textwidth]{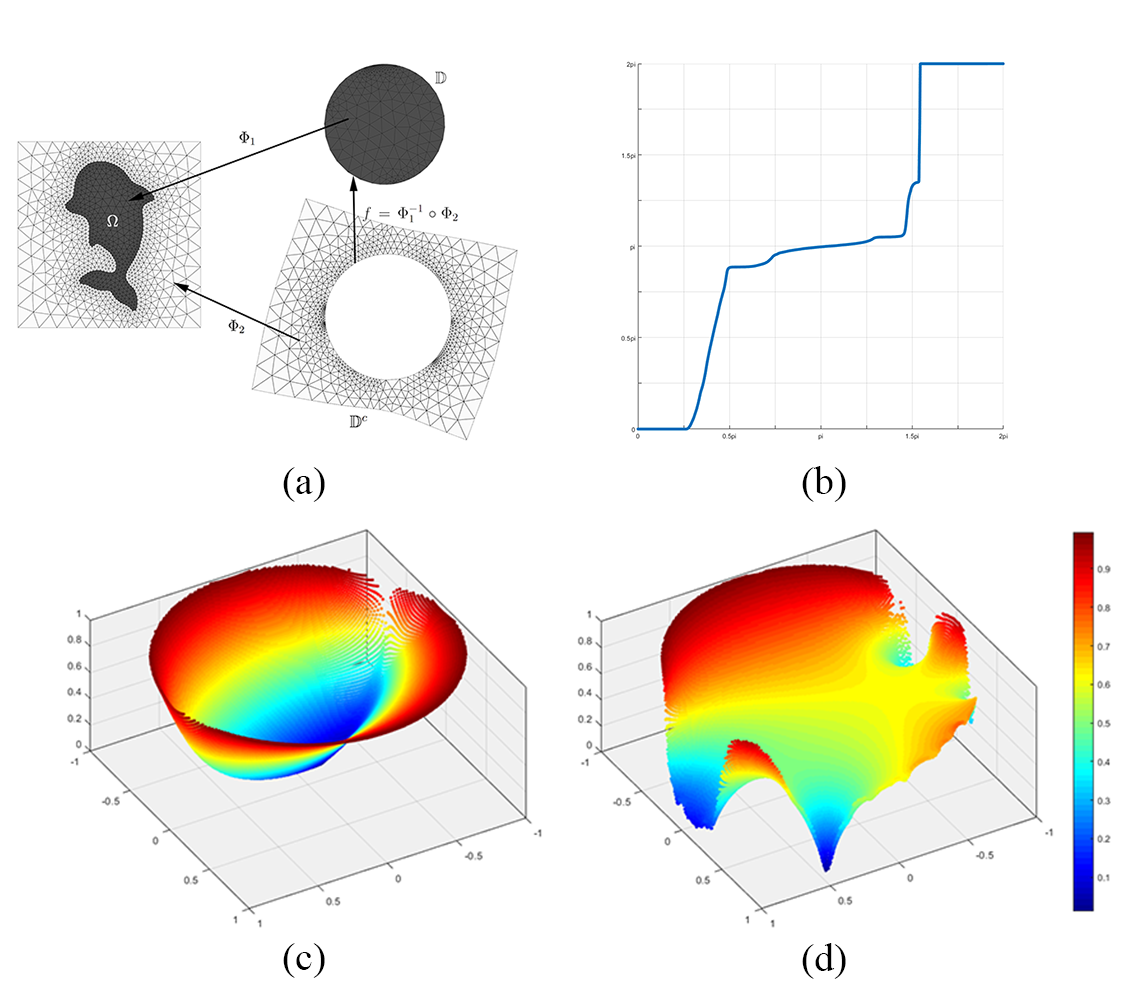}
    \end{center}
    \caption{The illustration of HBS. (a) shows the shape $\Omega$ and conformal maps $\Phi_1$ and $\Phi_2$; (b) is the conformal welding $f = \Phi_1^{-1} \circ \Phi_2$; (c) is the Harmonic extension $H$ of conformal welding $f$; (d) is the GHBS $B$ corresponding to $H$.}
    \label{fig: ghbs}
\end{figure}

\subsection{HBS}
Suppose $\Omega\subset \mathbb{C}$ is a Jordan domain and a quasicircle. Let $f = \Phi_1^{-1} \circ \Phi_2$ be the conformal welding of $\Omega$, where  $\Phi_1: \D \rightarrow \Omega$ and $\Phi_2: \D^c \rightarrow \Omega^c$ are the conformal mappings. The harmonic extension $H: \overline{\D} \to \C$ of $f$ is achieved by the Poisson integral on the unit circle. Then the Beltrami coefficient of $H$ is called \textit{Generalized Harmonic Beltrami Signature (GHBS)}
\begin{equation*}
    B(z):= \mu_H(z) = \frac{H_{\bar{z}}(z)}{H_z(z)}.
\end{equation*}
The process of obtaining GHBS is shown in \cref{fig: ghbs}.

When the $\Phi_2$ of GHBS satisfies $\Phi_2(\infty) = \infty$, we call this GHBS fixed at infinity, and then the equivalence relation of GHBS can be defined as
\begin{definition}
    Suppose two GHBS $B$ and $\tilde{B}$ are fixed at infinity, they are said to be {\it equivalent} if $B=\mu_H$ and $\tilde{B} = \mu_{\tilde{H}}$, where $H$ and $\tilde{H}$ are respectively the harmonic extensions of a diffeomorphism $f:\mathbb{S}^1\to \mathbb{S}^1$ and $\tilde{f} = M_1^{-1} \circ f \circ M_2$ with $M_1$ is a Mobi\"us transformation and $M_2$ is a rotation. In this case, we denote $B \sim \tilde{B}$. Also, the equivalence class of $B$ is denoted by $[B]$.
\end{definition}

We consider the space of GHBS equivalence classes $\mathcal{B} = \mathcal{B}_0 \,/ \sim$ to study the quotient space of shapes $\mathcal{S} = \mathcal{S}_0 \, / \sim_{\mathcal{S}}$, where $\mathcal{B}_0 = \{B:\mathbb{D}\to \mathbb{D} \mid B \text{ is a GHBS fixed at infinity} \}$, $\mathcal{S}_0 = \{\Omega \subset \C \mid \Omega \text{ is Jordan domain}\}$, $\Omega \sim_{\mathcal{S}} \bar{\Omega}$ iff $\bar{\Omega} = F(\Omega)$ and $F$ is composed of translation, rotation and scaling. The following theorem illustrates that GHBS is an effective representation.

\begin{theorem}\label{thm: one to one equivalence class}
    There is a one-to-one correspondence between $\mathcal{B}$ and $\mathcal{S}$. In particular, given $[B]\in \mathcal{B}$, its associated shape $\Omega$ can be determined up to a translation, rotation, and scaling.
\end{theorem}

The \textit{Harmonic Beltrami Signature (HBS)} $B$ is the unique representative of a GHBS equivalence class, which satisfies the following conditions:
\begin{gather}
    \int_{\partial \Omega} \Phi_1^{-1}(z) dz = 0,\\
    \label{eq: arg integral B is 0}\arg \int_\D B(z) dz = 0,\\
    \label{eq: arg in 0 and pi}\arg \int_\D \frac{B(z)}{z} dz \in [0, \pi).
\end{gather} 

Given shape $\Omega$, its HBS $B$ is uniquely determined and is invariant under rotation, translation, and scaling. The following theorem guarantees the geometric stability of HBS:
\begin{theorem}\label{thm: stability of HBS}
    Let $B_1$, $B_2$ be two HBS and $\Omega_1$ and $\Omega_2$ be the normalized shapes associated to $B_1$ and $B_2$ respectively. If
    $||B_1 - B_2||_{\infty} < \epsilon$, then the Hausdorff distance between $\Omega_1$ and $\Omega_2$ satisfies
    \begin{equation*}
        d_H(\Omega_1,\Omega_2)
        = \max \left(
        \max_{q \in \Omega_2} \min_{p \in \Omega_1} \abs{p-q},
        \max_{p \in \Omega_1} \min_{q \in \Omega_2} \abs{p-q}
        \right)
        < \frac{2M}{\pi}\epsilon
    \end{equation*}
    for some $M>0$.
\end{theorem}

Besides, the effective reconstruction algorithm from an HBS $B$ to the original shape $\Omega$ is also proposed and proved in \cite{linHarmonicBeltramiSignature2022}. This algorithm constructs a function $G: \C \rightarrow \C$ such that $G(\D) = \Omega$ and the Beltrami coefficient of $G$ is
$$
\mu_G = \begin{cases}
    B \text{ on } \D,\\
    0 \text{ on } \D^c.
\end{cases}
$$

\section{Proposed segmentation model}\label{section: main}
This section describes our proposed shape prior segmentation model guided by HBS.

\subsection{quasi-conformal topology preserving segmentation model}
Recall the quasi-conformal topology preserving segmentation proposed by Chan \cite{chanTopologyPreservingImageSegmentation2018}, which is the fundamentation of our HBS segmentation model. Suppose $D, D' \in \C$ are two regions, $I: D' \rightarrow \R$ is an image, and $J: D \rightarrow \R$ is a binary template image. $J$ is called the \textit{topological template image} defined as
\begin{equation}
    J(x) = \begin{cases}
        1, & x \in R,             \\
        0, & x \in D \setminus R,
    \end{cases}
\end{equation}
where $R \subset D$ is the object region of $J$. The key idea of this segmentation model is to find a quasi-conformal mapping $f_\mu: D \rightarrow D'$ which deforms the template $J$ to make it more similar to the given image $I$, then the segmented target region is $\Omega = f_\mu(R)$. Note that in most cases, we choose the rectangle domain $D = D'$.

This model can be formulated as the following energy functional:
\begin{equation}\label{eq: bc seg model}
    \min_{\mu} E_{\text{BC}}(\mu, c_1, c_2) = \min_{\mu} \int_D (I \circ f_\mu - J_{c_1,c_2})^2 + \alpha \abs{\mu}^2 + \beta \abs{\nabla \mu}^2 + \gamma \abs{u}^2 + \delta \abs{\nabla u}^2,
\end{equation}
where $u = f_\mu - Id$, $\mu$ is the Beltrami coefficient of $f_\mu$,  $\alpha, \beta, \gamma, \delta \ge 0$ are weight parameters and $J_{c_1, c_2}$ is a generalized template image defined as
\begin{equation}
    J_{c_1, c_2}(x) = \begin{cases}
        c_1, & x \in R,             \\
        c_2, & x \in D \setminus R.
    \end{cases}
\end{equation}
The first term of $E_{\text{BC}}$ measures the difference between the deformed image $I \circ f_\mu$ and the template image $J_{c_1, c_2}$. The second term limits the magnitude of $\mu$ since $\abs{\mu} < 1$ if and only if $f_\mu$ is quasi-conformal, which guarantees that $f_\mu$ is a diffeomorphism and that $\Omega$ and $R$ are the same under topological sense. The last three terms are the regularization terms controlling the smoothness of $f_\mu$. \cref{fig: hbs seg} shows an example of the segmentation process and gives intuitive views of the variables of this model.

\begin{figure}
    \begin{center}
        \includegraphics[width=\textwidth]{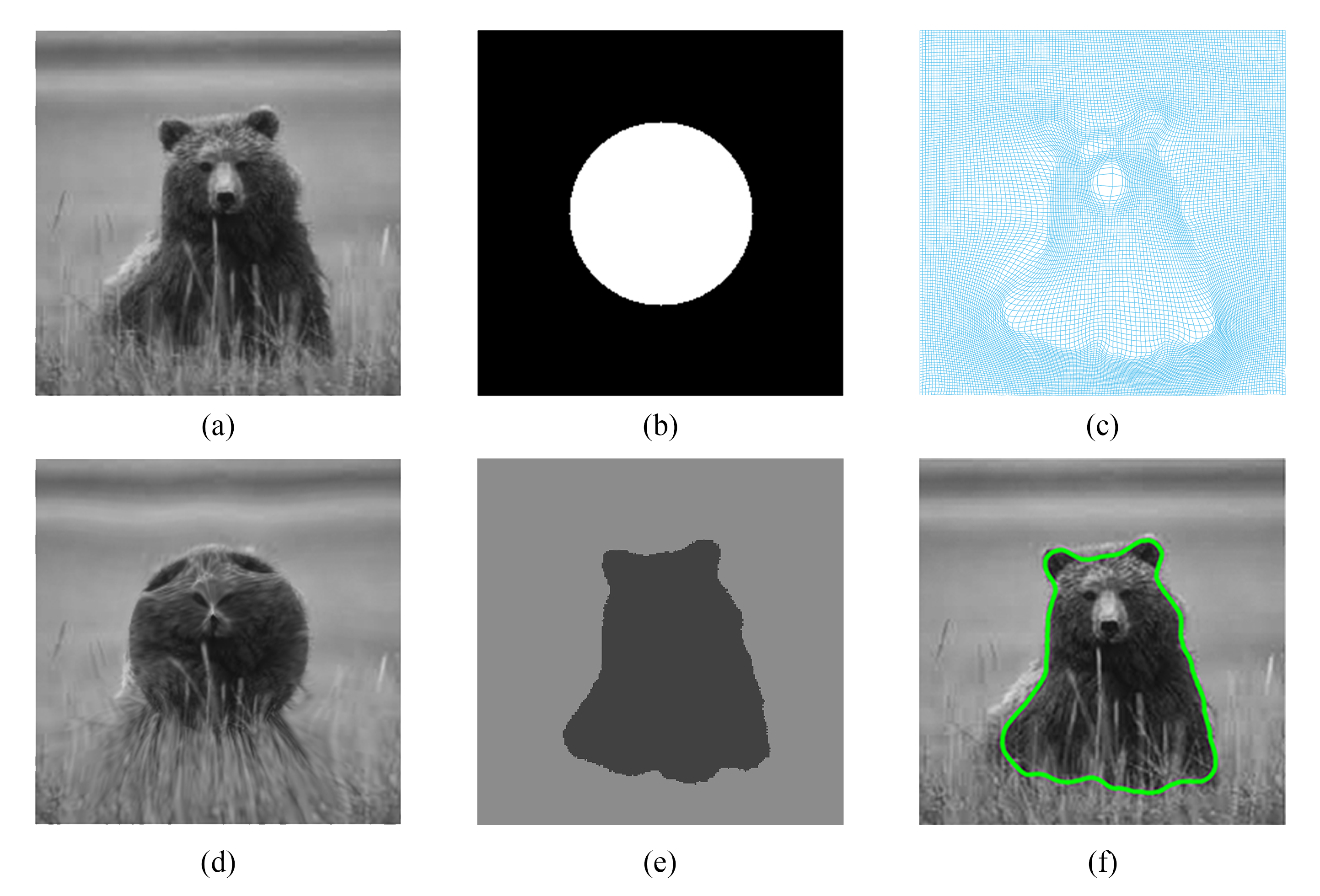}
    \end{center}
    \caption{(a) is the original image $I$; (b) is the initial template $J_{1,0}$; (c) is the deform map $f_\mu$, which transforms a standard square grid into the illustrated grid; (d) is $I \circ f_\mu$, we can find that the bear is almost put into the unit disk; (e) is deformed template $J_{c_1,c_2} \circ f_\mu^{-1}$, where $c_1$ and $c_2$ are determined by the mean color of $I \circ f_\mu |_\D$ and $I \circ f_\mu |_{\D^c}$ respectively; (f) is the segmentation result, the green line is the boundary of target domain $f_\mu(\D)$.}
    \label{fig: hbs seg}
\end{figure}

\subsection{HBS as shape prior}
The HBS is a signature of 2D simply connected shapes, and it represents the shape features by a complex function $B: \D \rightarrow \C$. More precisely, the HBS $B$ of given shape $\Omega_B \subset \C$ is the Beltrami coefficient of a \textit{HBS mapping} $G$, which is defined as
\begin{definition}\label{def: HBS mapping}
    A quasi-conformal mapping $G: \C \to \C$ is called HBS mapping if (1) its Beltrami coefficient $\mu_{G|_\D}$ is just the HBS of shape $\Omega_B = G(\D)$; (2) $G|_\D= \Phi \circ H$ where $\Phi: \D \to \Omega_B$ is conformal, $H: \D \to \D$ is harmonic; (3) $G|_{\D^c}$ is conformal.
\end{definition}

Indeed, the HBS provides a versatile and user-friendly shape representation. It eliminates the need for preprocessing tasks such as resizing and aligning, as well as the imposition of constraints on additional properties like convexity and curvature for input shapes.

The HBS $B$ is invariant under translation, rotation and scaling, which means it can be treated as a high-level shape prior with invariant geometric information. As is well-known, such transformations are widespread in image processing, and we have good reasons to believe that the HBS can effectively guide image segmentation.

In addition to those, HBS also possesses good stability. When two HBS are similar, their corresponding shapes only have minor differences, which is ensured by \cref{thm: stability of HBS}. With this property, we can effortlessly extract common features as prior knowledge from the contours of many objects belonging to the same category by HBS. Here, we adopt the most straightforward method of computing the average HBS. Suppose we have $N$ shapes $\Omega_1, \Omega_2, \cdots, \Omega_N$ then their corresponding HBS can be computed as $B_1, B_2, \cdots, B_N$. We take the average HBS $\bar{B} = \frac{1}{N} \sum^N_{i=1} B_i$ as the shape prior of this category. When only one shape $\Omega_1$ exists, we employ its HBS $B_1$ directly.
% \begin{equation*}
%     \bar{B} = \frac{1}{N} \sum^N_{i=1} B_i
% \end{equation*}

\subsection{HBS segmentation model}
The primary objective of the segmentation model \cref{eq: bc seg model} is to locate a suitable Beltrami coefficient $\mu$, then the corresponding quasi-conformal function $f_\mu$ and the target region $\Omega = f_\mu(R)$ can be determined. 
While the HBS $B$ of the given shape $\Omega_B$ is also a Beltrami coefficient of a HBS mapping $G$ such that $\Omega_B = G(\D)$. 
It is clear that the triplets $(B, G, \Omega_B)$ of the HBS and $(\mu, f_\mu, \Omega)$ of model \cref{eq: bc seg model} exhibit incredibly high correlation. 
Inspired by their intrinsic connection, we integrate the HBS into the quasi-conformal topology preserving segmentation model as shape prior and propose a novel HBS segmentation model. Leveraging the excellent attributes of HBS, we can obtain an improved segmentation result $\Omega$ that closely resembles the provided prior shape(s) $\Omega_B$ even when dealing with low-quality images.

In order to better integrate them, we have also made the following modifications:
\begin{enumerate}
    \item \textbf{Extend HBS:} $B$ is defined on $\D$ while $\mu$ is defined on $D$. To solve the difference in the domain of definition, we require the image domain $D \supset \D$ with $0$ as the center of $D$ and then extend $B$ to $D$ by $0$, that is
          \begin{equation}\label{eq: mu_B}
              \mu_B = \begin{cases}
                  B, & z \in \D,             \\
                  0, & z \in D \setminus \D.
              \end{cases}
          \end{equation}
    \item \textbf{Fix the template:} The shape can be represented as $\Omega_B = G(\D)$ in the HBS while the segmentation result is $\Omega = f_\mu(R)$ in model \cref{eq: bc seg model}. Therefore, we need to set the object region $R$ to be $\D$, and the template become
          \begin{equation}
              J_{c_1,c_2} = \begin{cases}
                  c_1, & z \in \D,             \\
                  c_2, & z \in D \setminus \D.
              \end{cases}
          \end{equation}

    \item \textbf{Add constraints:} With the above modifications, $\mu$, and $\mu_B$ are unified in form, which means they are both Beltrami coefficients of some quasi-conformal mappings on $D$. A basic assumption in shape prior segmentation is that the ideal segmentation result $\Omega^*$ should be similar to the prior shape $\Omega_B$. Moreover, \cref{thm: stability of HBS} provides a powerful tool to measure the distance between shapes by their HBS. However, this theorem is not suitable for $\norm{\mu - \mu_B}^2_2$, since $f_\mu$ only needs to be quasi-conformal in the original model \cref{eq: bc seg model} instead of a HBS mapping. Unfortunately, the high complexity of the HBS algorithm makes it exceedingly difficult to entirely confine $f_\mu$ within the space formed by HBS mappings. Consequently, we have chosen to add a crucial constraint term into energy functional
    \begin{equation}\label{eq: harmonic constrain of fmu}
        \int_D \abs{\Delta \text{imag} ( \ln \mu)}^2,
    \end{equation}
    where $\text{imag(.)}$ is the imaginary part of a complex number and $\Delta$ is the Laplace operator. This term comes from the Fotiadis et al. \cite{fotiadis2022beltrami} theorem:
    \begin{theorem}\label{thm: fotiadis}
        If the diffeomorphism $u: M \to N$ satisfies the Beltrami equation
        $$
        \dfrac{u_{\bar{z}}}{u_z} = e^{-2 \omega(z, \bar{z}) + i \phi(z, \bar{z})},
        $$
        where $\phi_{z \bar{z}} = 0$, then $N$ can be equipped with a conformal metric such that $u$ is a harmonic map and the curvature of $N$ is given by
        $$
        \tilde{K}_N = -\dfrac{2 \omega_{z \bar{z}}}{\sinh 2 \omega} e^{\psi},
        $$
        where $\psi$ is the conjugate harmonic function to $\phi$.
    \end{theorem}
    If $f_\mu$ is a HBS mapping, it can be decomposed as $f_\mu = \Phi \circ H$ where $\Phi$ is conformal and $H$ is harmonic. Then \cref{thm: fotiadis} shows that the imaginary part of $\ln \mu$ is harmonic. Hence, we minimize term \cref{eq: harmonic constrain of fmu} to make $f_\mu$ as close to an HBS mapping as possible.
\end{enumerate}

%% discussion abt ideal segmentation result

Therefore, the HBS segmentation model is formulated as follows:
\begin{equation}\label{eq: hbs seg model}
    \begin{split}
        \min_{\mu} E_{\text{HBS}}(\mu, B, c_1, c_2) =
        \min_{\mu} \int_D
        & (I \circ f_\mu - J_{c_1,c_2})^2
        + \alpha \abs{\mu}^2 + \beta \abs{\nabla \mu}^2
        + \gamma \abs{u}^2 \\
        & + \delta \abs{\nabla u}^2
        + \lambda \abs{\mu - \mu_B}^2 + \eta \abs{\Delta \text{imag} ( \ln \mu)}^2.
    \end{split}
\end{equation}
The first five terms directly come from model \cref{eq: bc seg model} and play the same roles as before. The sixth term measures the similarity of the segmentation result and the given shape prior HBS. Furthermore, the last term with a big penalty parameter $\eta$ is a soft constraint that $f_\mu$ is almost an HBS mapping.

The term $\int_D \abs{\mu - \mu_B}^2$ is crucial for proposed model \cref{eq: hbs seg model}. For a given image $I$, suppose that there is an ideal segmentation $\Omega^*$, whose HBS is $\mu^*$, then we have
\begin{equation}
    \norm{\mu - \mu^*}_2 \le \norm{\mu - \mu_B}_2 + \norm{\mu_B - \mu^*}_2.
\end{equation}
According to the basic assumption of shape prior segmentation, the given HBS $\mu_B$ is analogous to $\mu^*$, which means $\norm{\mu_B - \mu^*}_2^2$ is some small constant. Therefore, we can minimize the term $\norm{\mu - \mu_B}_2^2$ to make segmentation result $\Omega$ a good approximation of $\Omega^*$, which provides a simple yet effective method to utilize prior information for guiding the segmentation process. 

\begin{figure}
    \begin{center}
        \includegraphics[width=\textwidth]{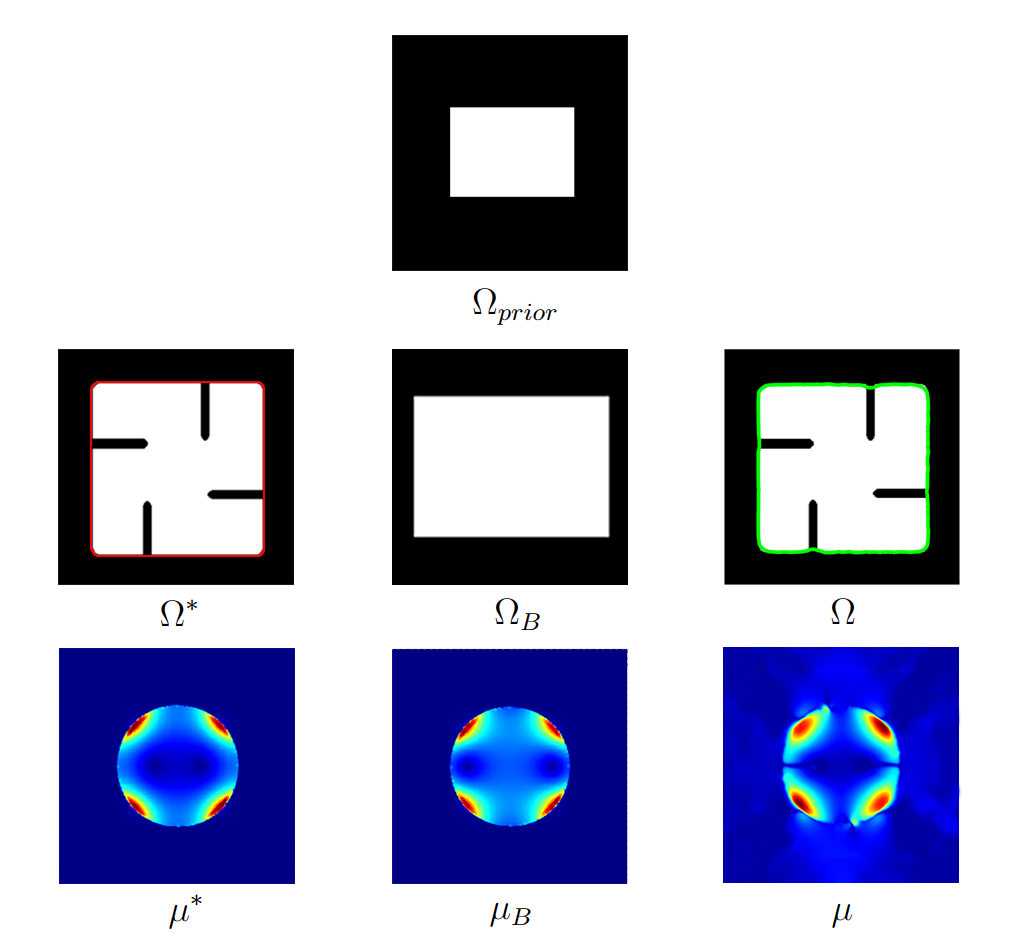}
    \end{center}
    \caption{The illustration of why the HBS can guide segmentation. \textbf{Left:} $\Omega^*$ is the ideal segmentation result, and $\mu^*$ is the HBS of $\Omega^*$. 
    \textbf{Middle:} $\Omega_{prior}$ is the reference shape, $\Omega_B$ is the shape obtained by translating, rotating and scaling $\Omega_{prior}$ to maximum the similarity to $\Omega^*$, and $\mu_B$ is the HBS of $\Omega_{prior}$ and $\Omega_B$. 
    \textbf{Right:} $\Omega$ is the segmentation result, and $\mu$ is the Beltrami coefficient of the deformation map corresponding to $\Omega$.}
    \label{fig: ill_mu}
\end{figure}

\cref{fig: ill_mu} provides a visible example demonstrating how minimizing $\norm{\mu - \mu_B}_2^2$ can induce a segmentation result $\Omega$ similar to the unknown ideal segmentation $\Omega^*$.
Several key points merit attention in this example:
\begin{enumerate}
    \item $\Omega^*$ and $\mu^*$ represent the unknown ideal segmentation;
    \item $\Omega_B$ is the shape obtained by translating, rotating and scaling $\Omega_{prior}$ to maximum the similarity to $\Omega^*$;
    \item In practice, $\Omega_B$ would never be computed but $\mu_B$ can be computed from $\Omega_{prior}$;
    \item Directly comparing $\Omega$ with $\Omega_{prior}$ or even $\Omega_B$ to achieve better results is not feasible since they visually differ from the ideal segmentation $\Omega^*$.
    % \item The HBS, Or the Beltrami coefficient more broadly, 
\end{enumerate}

The following theorem guarantees the existence of the minimizer of model \cref{eq: hbs seg model} over $\mu$.
\begin{theorem}\label{thm: existence}
    The energy functional $E_{\text{HBS}}$ in model \cref{eq: hbs seg model} has a minimizer in $\mathcal{A}^M_\epsilon \subset C^1(D)$, where
    \begin{equation*}
        \mathcal{A}^M_\epsilon = \{ \mu \in C^1(D) : \norm{D \mu}_\infty \le M, \norm{\mu}_\infty \le 1 - \epsilon \},
    \end{equation*}
    for some $M > 0$ and $\epsilon \in (0, 1)$.
\end{theorem}

\begin{proof}
    Following a similar argument in \cite{lam2014landmark}, we can show that $\mathcal{A}^M_\epsilon$ is Cauchy complete and bounded, and hence compact. The four terms containing $\mu$ are obviously continuous over $\mu$. Also, according to the Beltrami holomorphic flow and Bojarski theorem, the associated quasi-conformal map $f_\mu$ varies continuously (smoothly) under a continuous (smooth) variation of $\mu$. Hence, the other terms are continuous over $\mu$ as well. Since $E_\text{HBS}$ is continuous on the compact set $\mathcal{A}^M_\epsilon$, $E_\text{HBS}$ has a minimizer in it.
\end{proof}

Similar to that mentioned in \cite{siu2020image}, the smaller $\epsilon$ is, the more geometric distortion is allowed, and the deformed contour gets closer to the actual object boundary. However, small $\epsilon$ increases the difficulty of finding the minimizer, which is a trade-off between the accuracy and the efficiency. Meanwhile, the bigger $M$ is, the smoother the deformed contour is, but we get less accurate segmentation results. We can choose $\epsilon$ and $M$ according to the prior information of the given shape.

\section{Numerical algorithm}\label{section: algorithm}
In practice, a digital image consists of many pixels, which can be discrete using a triangular mesh $(V, E, F)$, where $V \subset D$ is the set of vertices, $E$ is the edges between vertices, and $F$ is the set of triangle faces formed by edges. 
The deformation function $f_{\mu_V}: V \rightarrow D$ is piecewise linear on each face and $f_{\mu_V}|_{\partial D} = Id$. The first derivative of $f_{\mu_V}$ is a piecewise constant, so the Beltrami coefficient $\mu_F: F \rightarrow \C$ is a constant on each face. Subsequently, $\mu_V: V \rightarrow \C$ is computed by averaging $\mu_F$ over the faces sharing the same vertex. A similar procedure is followed to compute $\mu_{B,V}$. Although the input images are only defined on $V$, we treat it as a continuous function $I, J: D \rightarrow \R$ through interpolation for the sake of convenience in the discussion. Hence, the deformed image $I \circ f_{\mu_V}: V \rightarrow \R$ can be achieved. As a result, the discrete HBS segmentation model is as follows:
\begin{equation}\label{eq: discrete hbs seg model}
    \begin{split}
        \min_{\mu_V} E_{\text{DHBS}}(\mu_V, B, c_1, c_2) =
        \min_{\mu_V} \sum_{v \in V}
        & (I \circ f_{\mu_V} - J_{c_1,c_2})^2
        + \alpha \abs{\mu_V}^2 + \beta \abs{\nabla \mu_V}^2
        + \gamma \abs{u_V}^2 \\
        & + \delta \abs{\nabla u_V}^2
        + \lambda \abs{\mu_V - \mu_{B,V}}^2 + \eta \abs{\Delta \text{imag}(\ln \mu_V)}^2.
    \end{split}
\end{equation}

However, this model implies a constraint $\mu_V = \frac{\partial f_{\mu_V}}{\partial \bar{z}} / \frac{\partial f_{\mu_V}}{\partial z}$, which presents a significant challenge during searching the minimizer. To address this, we disconnect the strong coupling between $\mu_V$ and $f_{\mu_V}$ and set another Beltrami coefficient $\nu_{V}$ without direct relation with $f_{\mu_V}$ as an alternative, which induces a weak HBS segmentation model:
\begin{equation}\label{eq: weak discrete hbs seg model}
    \begin{split}
         &\min_{\mu_V, \nu_V} E_{\text{WHBS}}(\mu_V, \nu_V, B, c_1, c_2) \\
        =& \min_{\mu_V, \nu_V} \sum_{v \in V}
        (I \circ f_{\mu_V} - J_{c_1,c_2})^2
        + \alpha \abs{\nu_V}^2 + \beta \abs{\nabla \nu_V}^2
        + \gamma \abs{u_V}^2 \\
        & + \delta \abs{\nabla u_V}^2
        + \lambda \abs{\nu_V - \mu_{B,V}}^2 + \eta \abs{\Delta \text{imag}(\ln \nu_V)}^2 + \tau \abs{\nu_V - \mu_V}^2.
    \end{split}
\end{equation}
Here we replace $\mu_V$ by $\nu_V$ in 2nd, 3th, 6th and 7th terms of model \cref{eq: discrete hbs seg model} and add a new term $\tau \abs{\nu_V - \mu_V}^2$ to handle the difference between $\mu_V$ and $\nu_V$, where $\tau > 0$ is a big penalty parameter. If $\nu_V = \mu_V$, $E_{\text{WHBS}}$ will degenerate to $E_{\text{DHBS}}$.

This substitution allows us to split the original problem into two parts, and we call them \textbf{$\mu$ subproblem}
\begin{equation}\label{eq: mu subproblem}
    \min_{\mu_V} E_{\mu}(\mu_V, c_1, c_2) =
    \min_{\mu_V} \sum_{v \in V}
    (I \circ f_{\mu_V} - J_{c_1,c_2})^2
    + \gamma \abs{u_V}^2
    + \delta \abs{\nabla u_V}^2
    ,
\end{equation}
and \textbf{$\nu$ subproblem}
\begin{equation}\label{eq: nu subproblem}
    \begin{split}
    \min_{\nu_V} E_{\nu}(\mu_V, \nu_V, B) =
    &\min_{\nu_V} \sum_{v \in V}
    \alpha \abs{\nu_V}^2 + \beta \abs{\nabla \nu_V}^2
    + \lambda \abs{\nu_V - \mu_{B,V}}^2 \\
    &+ \eta \abs{\Delta \text{imag}(\ln \nu_V)}^2
    + \tau \abs{\nu_V - \mu_V}^2
    .
    \end{split}
\end{equation}

After such separation, these two subproblems have gained clearer and more practical meanings. The essence of the $\mu$ subproblem lies in determining the deformation mapping $f_{\mu_V}$ according to the pixel level difference and additional regularizations, then $\mu_V$ is achieved. Conversely, the target of the $\nu$ subproblem is to normalize the Beltrami coefficient $\mu_V$ obtained from the previous step under shape prior information for further segmentation. The soft constraint term $\abs{\nu_V - \mu_V}$ bridges these two processes. Therefore, we can solve these two subproblems alternatively until they converge, which then gives the solution of the original optimization problem \cref{eq: discrete hbs seg model}. 

More specifically, suppose $\mu_{V,n}$ and $\nu_{V,n}$ is obtained in $n$-th iteration, then $\mu_{V,n+1}$ and $\nu_{V,n+1}$ can be computed by
\begin{equation}
    \mu_{V,n+1} = \argmin_{\mu_V} E_\mu(\mu_V, c_1, c_2),
\end{equation}
\begin{equation}
    \nu_{V,n+1} = \argmin_{\nu_V} E_\nu(\mu_{V,n+1},\nu_V, B).
\end{equation}
The total algorithm is summarized in \cref{alg: main}.

\subsection{Solving $\mu$ subproblem}\label{section: mu subproblem}
The target of $\mu$ subproblem is to find an optimal $\mu_V$,
but there are two additional input variables $c_1$ and $c_2$ in $E_\mu$. Compared to solving for multiple variables simultaneously, it is simpler and more common to determine $c_1$ and $c_2$ in advance. Given $f_{\mu_{V,n}}$, then $\mu_{V,n}$ is fixed and we can compute the minimizers explicitly by
\begin{equation}\label{eq: mu c1 c2}
    c_1 = \frac{\sum_{v \in V \cap \D} I \circ f_{\mu_{V,n}}(v)}{\sharp (V \cap \D)}, \quad
    c_2 = \frac{\sum_{v \in V \setminus \D} I \circ f_{\mu_{V,n}}(v)}{\sharp (V \setminus \D)},
\end{equation}
where $\sharp(.)$ means the number of vertices inside the set. Except in necessary cases, we will omit all occurrences of $c_1$ and $c_2$ in the following discussion.

Even \cref{thm: existence} ensure the existence of minimizer, excessively large deformations can still burden the algorithm, lead to fluctuating segmentation results, and take long computation time. In $n$-th iteration, instead of directly calculating $f_{\mu_{V,n+1}}$, we attempt to compute a smaller deformation $g_{n+1}$ based on the previous segmentation result $f_{\mu_{V,n}}$ 
%and a sequence of such deformations can also represent the desired $f_{\mu_{V,n+1}}$
\begin{equation}
    f_{\mu_{V,n+1}} = g_{n+1} \circ f_{\mu_{V,n}} = g_{n+1} \circ \cdots \circ g_2 \circ f_{\mu_{V,1}},
\end{equation}
where $g_k: D \rightarrow D$. Since $f_{\mu_{V,n}}$ is quasi-conformal then bijective, $f_{\mu_{V,n}}^{-1}$ exists, we  have
\begin{equation}\label{eq: mu 1st term approx v}
    \sum_{v \in V} (I \circ g_{n+1} \circ f_{\mu_{V,n}} - J)^2
    = \sum_{v \in f_{\mu_{V,n}}(V)} (I \circ g_{n+1} - J \circ f_{\mu_{V,n}}^{-1})^2
    \approx \sum_{v \in V} (I \circ g_{n+1} - J_n)^2,
\end{equation}
where $J_n = J \circ f_{\mu_{V,n}}^{-1}$ is the deformed template image by $f_{\mu_{V,n}}$ and is the temporary segmentation result of after $n$ iterations. 
The approximate equality is due to the acceptable interpolation error, which decreases as the image resolution increases. Furthermore, this term can be expanded by a first-order Taylor series
\begin{equation}\label{eq: mu 1st term taylor}
    \sum_{v \in V} (I \circ g_{n+1} - J_n)^2
    \approx \sum_{v \in V} \abs{I - J_n + \nabla I \cdot h_{n+1}}^2,
\end{equation}
where $h_{n+1} = g_{n+1} - Id: D \to \C$, $\nabla I: D \to \C$ is the discrete gradient of $I$ and $\cdot$ is an operator for two complex number such that $(a+bi) \cdot (c+di) = ac+bd$.

Since $u_{V,n+1} = h_{n+1} \circ f_{\mu_{V,n}} + f_{\mu_{V,n}} - Id =  h_{n+1} \circ f_{\mu_{V,n}} + u_{V,n}$, the same interpolation approximation method in \cref{eq: mu 1st term approx v} can be also applied to the other two regular terms of $E_\mu$ in \cref{eq: mu subproblem}. Finally, $\mu$ subproblem can be rewritten in form about $h$:
\begin{equation}
    \min_h E_\mu (h)
    \approx \min_h \sum_{v \in V} \abs{I - J_n + \nabla I \cdot h}^2
    + \gamma \abs{h}^2
    + \delta \abs{\nabla h}^2.
\end{equation}
This is a quadratic optimization problem and can be solved according to Euler-Lagrange equation
\begin{equation}\label{eq: mu problem pde}
    (I-J_n+ \nabla I \cdot h) \nabla I + \gamma h - \delta \Delta h = 0,
\end{equation}
subject to $h = 0$ on $\partial D$. We mark the solution of \cref{eq: mu problem pde} in $n$-th iteration as $h_{n+1}$ and the deform function is
\begin{equation}\label{eq: mu deform function f}
    f^*_{\mu_{V,n+1}} = (h_{n+1} + Id) \circ f_{\mu_{V,n}}.
\end{equation}
Then the Beltrami coefficient $\mu_{V,n+1}$ can be computed by
\begin{equation}\label{eq: mu mu}
    \mu_{V,n+1}(v) = \frac{\partial f^*_{\mu_{V,n+1}}}{\partial \bar{z}}(v) \bigg / \frac{\partial f^*_{\mu_{V,n+1}}}{\partial z}(v),
\end{equation}
To avoid some numerical errors, we designate $\mu_{V,n}(v) = 0$ if $\frac{\partial f^*_{\mu_{V,n}}}{\partial z}(v) = 0$ for some $v \in V$.

Besides, we also need $\abs{\mu_{V,n+1}(v)} < 1$ for all $v \in V$ to guarantee the deform function is quasi-conformal, so an additional truncation $T$ is applied to $\mu_{V,n+1}$:
\begin{equation}\label{eq: mu truncate}
    T(\mu)(v) = \begin{cases}
        \mu(v),                                  
        & \abs{\mu(v)} < 1 - \epsilon,   \\
        \dfrac{1 - \epsilon}{\abs{\mu(v)}}\mu(v), 
        & \abs{\mu(v)} \ge 1 - \epsilon,
    \end{cases}
\end{equation}
where $\epsilon$ is a small positive number. So far, we have addressed the $\mu$ subproblem and obtained the desired $\mu_{V,n+1}$.

\subsection{Solving $\nu$ subproblem}
As mentioned before, $\nu$ subproblem \cref{eq: nu subproblem} essentially involves refining the solution $\mu_{V,n+1}$ obtained from the previous step. 
For any complex function $p: D \rightarrow \C$ and $t \in \R$, we consider
\begin{equation}\label{eq: derivative of E_nu}
    \begin{split}
        &\frac{d}{dt}\Big|_{t=0} E_{\nu}(\mu_V, \nu_V + tp , B)  \\
    =& \sum_{v \in V}\frac{d}{dt}\Big|_{t=0} 
    \Big(
        \alpha \abs{\nu_V + tp}^2 
        + \beta \abs{\nabla (\nu_V+tp)}^2
        + \lambda \abs{\nu_V - \mu_{B,V} + tp}^2\\
    &\quad + \tau \abs{\nu_V - \mu_V + tp}^2
    + \eta \left(\Delta \frac{1}{2i}
    \left(\ln (\nu_V + tp) - \overline{\ln (\nu_V + tp)}\right)\right)^2
     \Big) \\
    =& \sum_{v \in V} 2((\alpha+\lambda+\tau) \nu_V - \lambda\mu_{B,V} - \tau \mu_V) \cdot p + 2 \beta \nabla \nu_V \cdot \nabla p \\
    &\quad - \frac{\eta}{4} \frac{d}{dt}\Big|_{t=0} \left(\Delta \ln (v+tp) - \Delta \overline{\ln(v+tp)}\right)^2 \\
    =& \sum_{v \in V} 2\left((\alpha+\lambda+\tau) \nu_V - \lambda\mu_{B,V} - \tau \mu_V\right) \cdot p 
    + 2 \beta \nabla \nu_V \cdot \nabla p \\
    &\quad - \frac{\eta}{2} \left(\Delta \ln \nu_V - \Delta \overline{\ln \nu_V}\right)
    \left(\Delta \frac{p}{\nu_V} - \Delta \frac{\overline{p}}{\overline{\nu_V}}\right).
    \end{split}
\end{equation}

Let $\phi = \text{imag}(\ln \nu_V)$, then $\phi: V \to \R$ is the argument of $\nu_V$ and $\ln \nu_V - \overline{\ln \nu_V} = 2i \phi$. 
Let $p = p_x + i p_y$ and $\nu_V = \nu_x + i \nu_y$, we have
\begin{equation}
    \frac{p}{\nu_V} - \frac{\overline{p}}{\overline{\nu_V}} 
    = \frac{p \overline{\nu_V} - \overline{p} \nu_V}{\abs{\nu_V}^2} 
    = 2i\frac{-p_x \nu_y + p_y \nu_x}{\abs{\nu_V}^2} 
    = 2i \frac{(i\nu_V) \cdot p}{\abs{\nu_V}^2} 
    = 2i \frac{i}{\overline{\nu_V}} \cdot p.
\end{equation}
According to the Green's first identity, we have
\begin{equation*}
    \begin{split}
        &\sum_{v \in V} \left(\Delta \ln \nu_V - \Delta \overline{\ln \nu_V}\right)
        \left(\Delta \frac{p}{\nu_V} - \Delta \frac{\overline{p}}{\overline{\nu_V}}\right)
        = -4 \sum_{v \in V} \Delta \phi \Delta\left(\frac{i}{\overline{\nu_V}} \cdot p\right) \\
        = & -4 \left( \sum_{v \in V} \Delta^2 \phi \left(\frac{i}{\overline{\nu_V}} \cdot p\right) 
        + \sum_{v \in \partial V} \left(\nabla \left(\frac{i}{\overline{\nu_V}} \cdot p\right) \cdot \vec{n}\right) \Delta \phi 
        - \sum_{v \in \partial V} \left(\nabla (\Delta \phi )\cdot \vec{n}\right) \left(\frac{i}{\overline{\nu_V}} \cdot p\right)\right)
    \end{split}
\end{equation*}
where $\vec{n}$ is the normal vector of $\partial V$. Since $\nu_V = 0$ always holds on $\partial V$, we have $p = 0$ and $\phi = 0$ on $\partial V$. Therefore, the last two terms can be moved, and we get
\begin{equation}
    \begin{split}
        \sum_{v \in V} \left(\Delta \ln \nu_V - \Delta \overline{\ln \nu_V}\right)
        \left(\Delta \frac{p}{\nu_V} - \Delta \frac{\overline{p}}{\overline{\nu_V}}\right)
        = -4 \sum_{v \in V}  \frac{ i \Delta^2 \phi}{\overline{\nu_V}} \cdot p.
    \end{split}
\end{equation}
Put it back to \cref{eq: derivative of E_nu}, and we get gradient of $E_\nu$ with respect to $\nu_V$ as
\begin{equation}\label{eq: nu gradient}
        \nabla_{\nu_V} E_{\nu}(\mu_V, \nu_V , B)  
    = 2 \sum_{v \in V} \left((\alpha+\lambda+\tau) \nu_V -\beta \Delta \nu_V - \lambda\mu_{B,V} - \tau \mu_V 
    +\eta \frac{ i \Delta^2 \phi}{\overline{\nu_V}} \right).
\end{equation}

After $\nabla_{\nu_V} E_\nu$ achieved, gradient desent method  is applied to compute $\nu_{V,n+1}$ in $n$-iteration of solving subproblem \cref{eq: nu subproblem}, which is summarized in \cref{alg: nu}.

\begin{algorithm}[H]
    \caption{$\nu$ subproblem algorithm in $n$-th iteration}
    \label{alg: nu}
    \begin{algorithmic}
        \STATE \textbf{Inputs:} triangle mesh $(V, E, F )$ on Rectangle domain $D \subset \C$, $\mu_{V,n+1}$, $\mu_{B,V}$, max iteration times $K$, stop precision $\epsilon$ and other model parameters.
        \STATE \textbf{Initialize:} Let $k=1$ and $\nu_{V,n+1}^1 = \mu_{V,n+1}$.
        \WHILE{$k < K$}
        \STATE Compute the argument $\phi$ of $\nu_{V,n+1}^k$.
        \STATE Compute $\nabla E_{\nu} = \nabla_{\nu_V} E_\nu(\mu_{V,n+1},\nu_{V,n+1}^k, B)$ by \cref{eq: nu gradient}.
        \STATE Let $\nu_{V,n+1}^{k+1} = T(\nu_{V,n+1}^k - t \, \nabla E_{\nu}) $, where truncation function $T$ is defined in \cref{eq: mu truncate}.
        \IF{$\norm{\nu_{V,n+1}^{k+1} - \nu_{V,n+1}^k}_2 < \epsilon$}
        \STATE Break.
        \ENDIF
        \STATE Let $k = k+1$.
        \ENDWHILE
        \RETURN $\nu_{V,n+1} = \nu_{V,n+1}^{k+1}$.
    \end{algorithmic}
\end{algorithm}

By the LBS method, we reconstruct the corresponding refined deform function $f_{\mu_{V,n+1}}$ from $\nu_{V,n+1}$. That is, solve the following PDE:
\begin{equation}
    \begin{split}
         \dfrac{\partial f_{\mu_{V,n+1}}}{\partial \bar{z}} \bigg /  \dfrac{\partial f_{\mu_{V,n+1}}}{\partial z}
         = \nu_{V,n+1} &\text{ on } \quad D, \\
        f_{\mu_{V,n+1}} = Id & \text{ on } \quad \partial D.
    \end{split}
\end{equation}
The solution $f_{\mu_{V,n+1}}$ induces the temporary segmentation result $J_{n+1} = J \circ f_{\mu_{V,n+1}}^{-1}$ and then is used in next iteration. Finally, we get the desired segmentation result $f_{\mu_{V,n+1}}(\D)$ when algorithm converges.

\begin{algorithm}[H]
    \caption{HBS segmentation algorithm}
    \label{alg: main}
    \begin{algorithmic}
        \STATE \textbf{Inputs:} Rectangle domain $D \subset \C$, the image to be segmented $I: D \rightarrow \R$, unit disk template $J: D \rightarrow \R$, HBS $B: \D \rightarrow \C$, max iteration times $N$, stop precision $\epsilon$ and other model parameters.
        \STATE \textbf{Initialize:} Build triangle mesh $(V, E, F)$ on $D$ and compute $\mu_{B,V}$ form $B$.
        \STATE Let $n=1$ and $f_{\mu_{V,1}} = Id$.
        \WHILE{$n < N$}
        \STATE Compute $c_1$ and $c_2$ by \cref{eq: mu c1 c2} and then $J_n = J \circ f_{\mu_{V,n}}$.
        \STATE Solve \cref{eq: mu problem pde} and get $h_{n+1}$.
        \STATE Compute $f^*_{\mu_{V,n+1}}$ by \cref{eq: mu deform function f}.
        \STATE Compute $\mu_{V,n+1}$ by \cref{eq: mu mu}.
        \STATE Solve $\nu_{V,n+1}$ as \cref{alg: nu}.
        \STATE Compute $f_{\mu_{V,n+1}}$ by LBS method.
        \IF{$\norm{f_{\mu_{V,n+1}} - f_{\mu_{V,n}}}_2 < \epsilon$}
        \STATE Break.
        \ENDIF
        \STATE Let $n = n+1$.
        \ENDWHILE
        \RETURN Segmentation result $f_{\mu_{V,n+1}}(\D)$.
    \end{algorithmic}
\end{algorithm}

\section{Experimental result}\label{section: exp}
In this section, we demonstrate the effectiveness of the proposed HBS segmentation model through different experimental results. Our experiments are implemented using MATLAB R2014a running on 4-way Intel(R) Xeon(R) Gold 6230 processors with 80 cores at 2.10GHz base frequency and 1024 GB RAM under Ubuntu 18.04LTS 64-bit operating system.

\subsection{Synthetic images}\label{section: exp: synthetic image}
Deformation of the target object due to various reasons, such as image corruption, object occlusion, etc., is pervasive. In such cases, prior information can greatly assist us in making accurate segmentations. By utilizing HBS as the shape prior, the proposed segmentation model is aware of the approximate shape of the target object beforehand, leading to improved performance.

To illustrate this more intuitively, we have constructed a series of simple synthetic images to validate our model and display the results in \cref{fig: synthetic image}. 
For each row, the 1st column is the original shape, the 2nd column is the template shape, the 3rd column is the HBS corresponding to the template shape, the 4th column is the segmentation result without HBS, and the 5th column is the segmentation result of our proposed model. Before analyzing the results, it is essential to clarify some points, and we will use them in the following experiments unless otherwise specified.
\begin{enumerate}
    \item The template provided in the 2nd column is not $J$ in our model \cref{eq: hbs seg model}, since $J$ has been fixed as the unit disk. Here, the template shape will not be fed into our model and is solely used to compute HBS $\mu_{B,V}$ in the 3rd column. We show them here to help our humans visually understand prior information.
    \item It is the magnitude of $\mu_{B,V}$ shown in the 3rd column to represent HBS.
    \item The green line in the 4th and 5th columns is the segmentation boundary.
    \item The segmentation result without HBS is obtained by setting $\lambda = \eta = 0$ of model \cref{eq: hbs seg model}, which then is equivalent to the quasi-conformal topology preserving segmentation model \cref{eq: bc seg model}. We also use the unit disk as $J$.
\end{enumerate}

Let us look back on \cref{fig: synthetic image}, from which we can preliminarily obtain the following information:
\begin{enumerate}
    \item The HBS played a significant role in the low-quality image segmentation. We made some modifications to basic geometric shapes, resulting in them lacking certain parts, like subfigures (a1), (c1), (d1), and (e1), or having additional parts, like the bottom side and right side of subfigure (b1). We use such modifications as damage and occlusion to represent image degradation. Under the guidance of HBS, our model essentially disregarded the impact of these variations and provided satisfactory segmentation, which almost matched the given template. Meanwhile, the segmentation result without HBS is much more faithful to the actual shape boundaries.
    \item The proposed model can segment both the simply connected and multiply connected images and yield simply connected results. Although the HBS is the signature of simply-connected shapes, our model works well on multi-connected images in (c1) and (d1).
    \item The proposed model does not impose any requirements on the position, size, or orientation of the target object in the image. The HBS is invariant under translation, rotation, and scaling, and our model inherits this feature. With the same HBS, (a1) and (b1) are rectangles of different sizes, and our model can correctly segment them. Similarly, (c1) and (d1) are triangles in different orientations.
    \item The proposed model tolerates reasonable discrepancies between the prior information and the actual image. In experiment (e), the segmentation result is acceptable when the template is a circle, and the image is close to an ellipse.
    \item The proposed model effectively utilizes the prior information provided by the given template rather than simply extracting some basic geometric properties. In experiment (b), our model handles convex and concave parts at the same time.
    \item The final segmentation results can not strictly align with the provided templates but are inevitably influenced by the actual shape of the target region. For example, as shown in experiment (b), the rectangle in (b1) has notches at the top and left while protruding at the bottom and right. Therefore, the green curve representing the segmented boundary in (b5) is inward or outward at corresponding positions. Such situations are also noticeable in all other experiments.
\end{enumerate}

\begin{figure}
    \begin{center}
        \includegraphics[width=15.5cm]{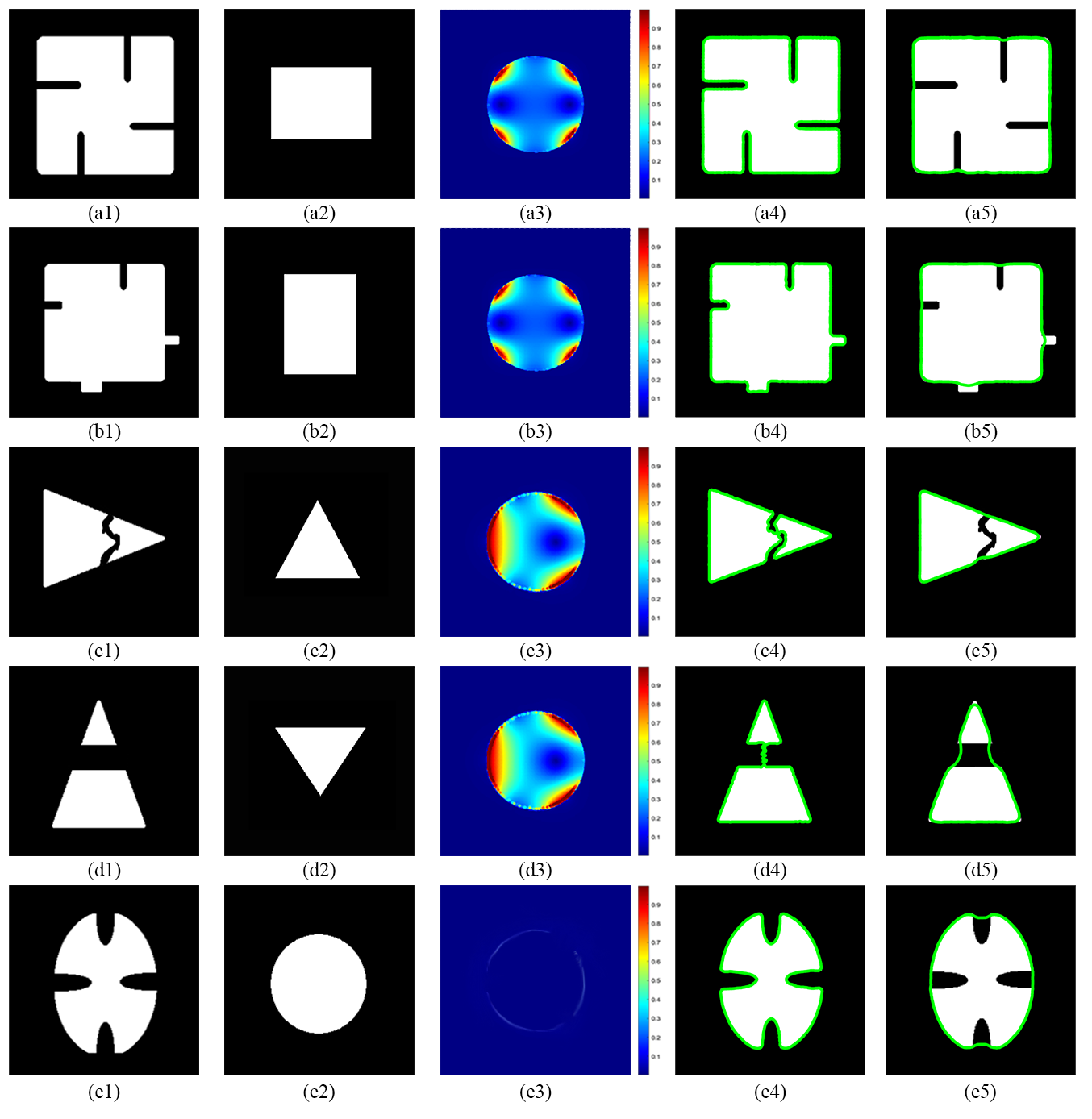}
    \end{center}
    \caption{Segmentation results of binary images. For each row, the sequence from left to right includes the original image, template shape, template's HBS, the segmentation result without HBS, and the segmentation result with HBS.}
    \label{fig: synthetic image}
\end{figure}

We will further observe the impact of different templates on the segmentation results from \cref{fig: seg under diff temps}. The meaning of each column is the same with \cref{fig: synthetic image}. 
The original image is similar to a hexagon, and we used hexagon (a2), round (b2), diamond (c2), and square (d2) as templates to segment it, with the results shown in (a5), (b5), (c5) and (d5) respectively. We can find that all segmentation results preserve some shape characteristics of their corresponding templates, but their performances differ significantly. 
The (a5) is the best, and (b5) is still acceptable, but (c5) and (d5) are far from accurate. That indicates our proposed method is highly sensitive to the given template. A precise template provides a perfect segmentation result, while an unsuitable template leads to a terrible one. On the other hand, it also implies that our model has a certain level of ability to understand prior information.

\begin{figure}
    \begin{center}
        \includegraphics[width=15.5cm]{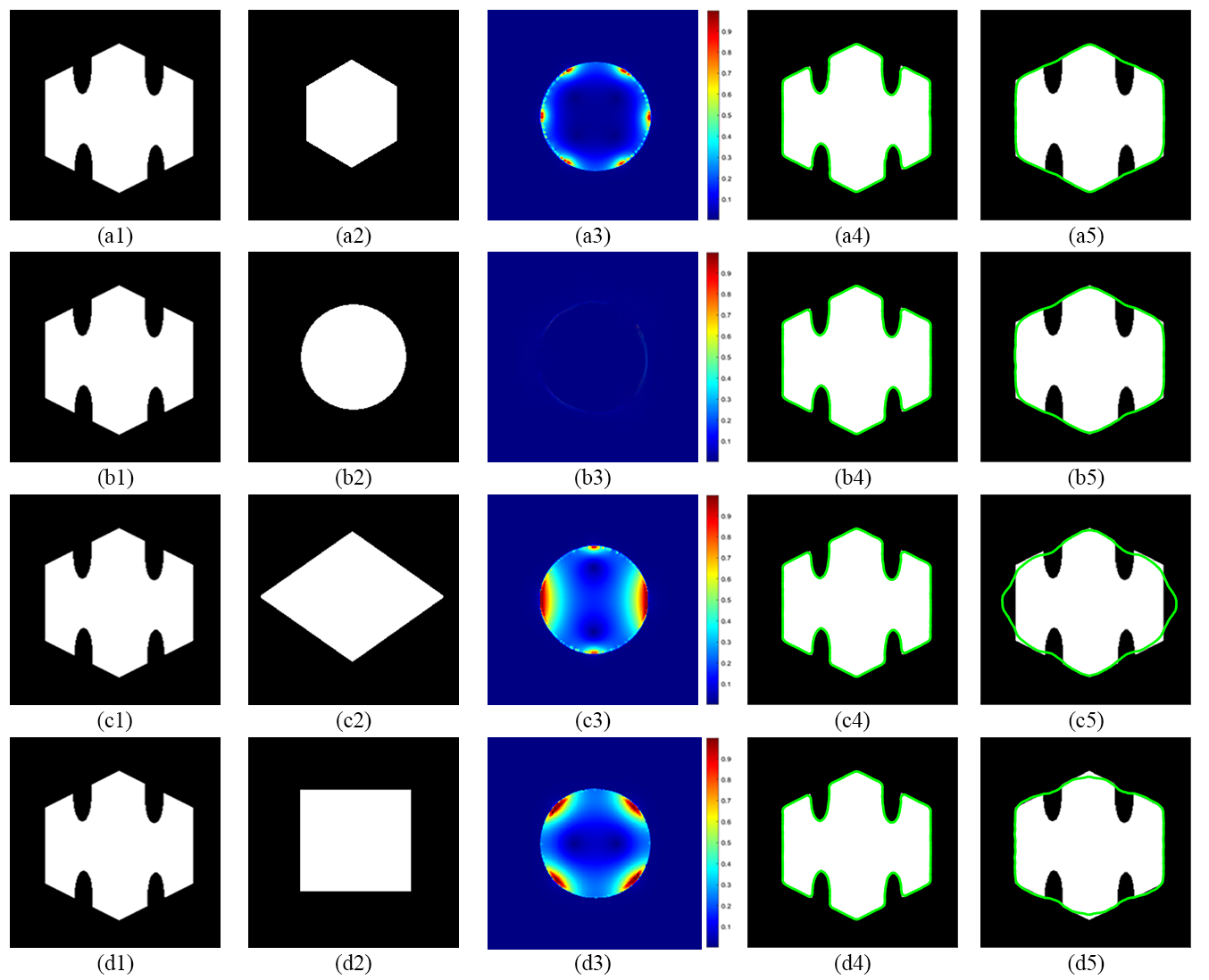}
    \end{center}
    \caption{Segmentation results under different templates. For each row, the sequence from left to right includes the original image, template shape, template's HBS, the segmentation result without HBS, and the segmentation result with HBS.}
    \label{fig: seg under diff temps}
\end{figure}

\subsection{Natural images}
The experimental results on synthetic images demonstrate that our model can effectively utilize prior information to segment partially damaged target objects. However, we are far from satisfied with this and have extended the model's application to real-world images.

\cref{fig: natural image} presents some results on natural images. As before, the original image, template, HBS, segmentation without HBS, and segmentation with HBS are sequentially displayed in columns 1 to 5 of each row. 

The first row (a1) is a grayscale image with a bear in grass. Without HBS, the model only relies on the intensity of grayscale values to determine the shape boundaries (a4), which results in the top boundary being limited to the brighter forehead and the bottom boundary appearing jagged due to the grass. When HBS is specified, even if the circle is not a suitable template, it helps our proposed HBS segmentation model generate a much smoother and more complete result (a5). The other three images are color images, which must be compressed RGB to grayscale value before applying our model. 
The result (b5) demonstrates that no matter how many pieces the target object is separated into, our model can segment it as a simply connected entity, just like it does on synthetic images. The experiment (c) represents a similar but more practical scenario, occlusion, and our model successfully and accurately identifies the white paper behind pens in (c5). In (d5), our model also locates the signboard partially covered under the dust.

\begin{figure}
    \begin{center}
        \includegraphics[width=15.5cm]{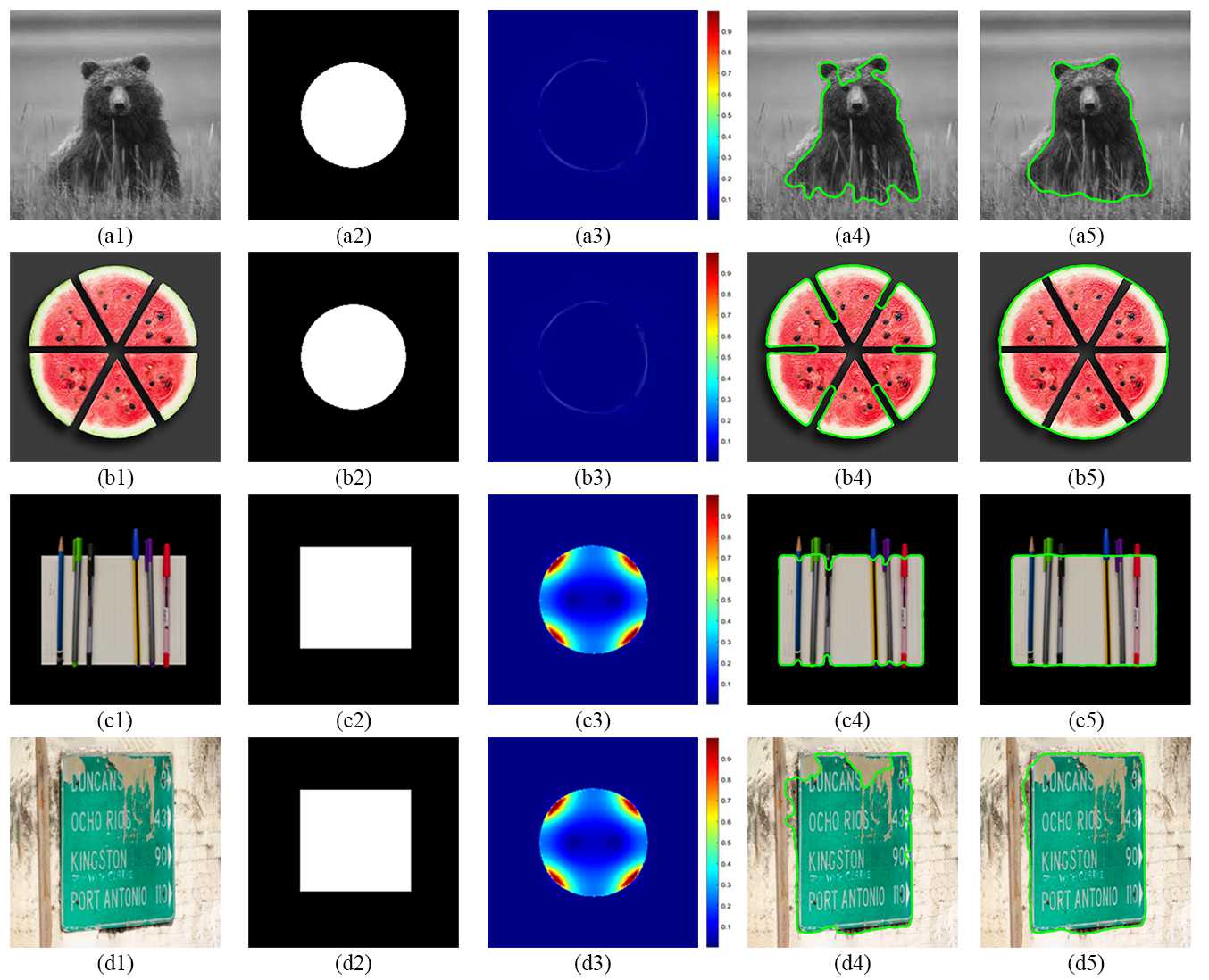}
    \end{center}
    \caption{Segmentation results of natural images. For each row, the sequence from left to right includes the original image, template shape, template's HBS, the segmentation result without HBS, and the segmentation result with HBS.}
    \label{fig: natural image}
\end{figure}

\subsection{Segmentation process}
The proposed method utilizes an iterative process to progressively deform an initial circular segmentation template into the final segmentation result, as described in \cref{section: mu subproblem}. This iterative deformation procedure was experimentally investigated, with results demonstrated in \cref{fig: seg process}. The segmentation progression on images of a bear and a brain is shown in the first and second rows, respectively. The initial circular template position is depicted in green in the first column. The intermediate temporary segmentation results after 5, 10, 15, and 20 iterations are displayed in green in the subsequent columns.

This experiment clearly exhibits the smooth deformation of the segmentation results across iterations, supporting the stability of the proposed iterative algorithm. Additionally, the convergence rate of the algorithm depends on several factors, including image complexity, resolution, HBS selection, step size, initial template position, and so on. For instance, the bear segmentation nears convergence after 15 iterations, while the brain segmentation still demonstrates significant differences between the 15th and 20th iterations. The algorithm typically converges within approximately 20 iterations under a suitable parameter setting, yielding satisfactory segmentation outcomes. 
% Further analysis of the impact of parameter selections on convergence rate could illuminate techniques for acceleration. 
In summary, this analysis of the iterative deformation process provides insights into the stability and convergence properties of the proposed HBS segmentation approach.

\begin{figure}
    \begin{center}
        \includegraphics[width=\textwidth]{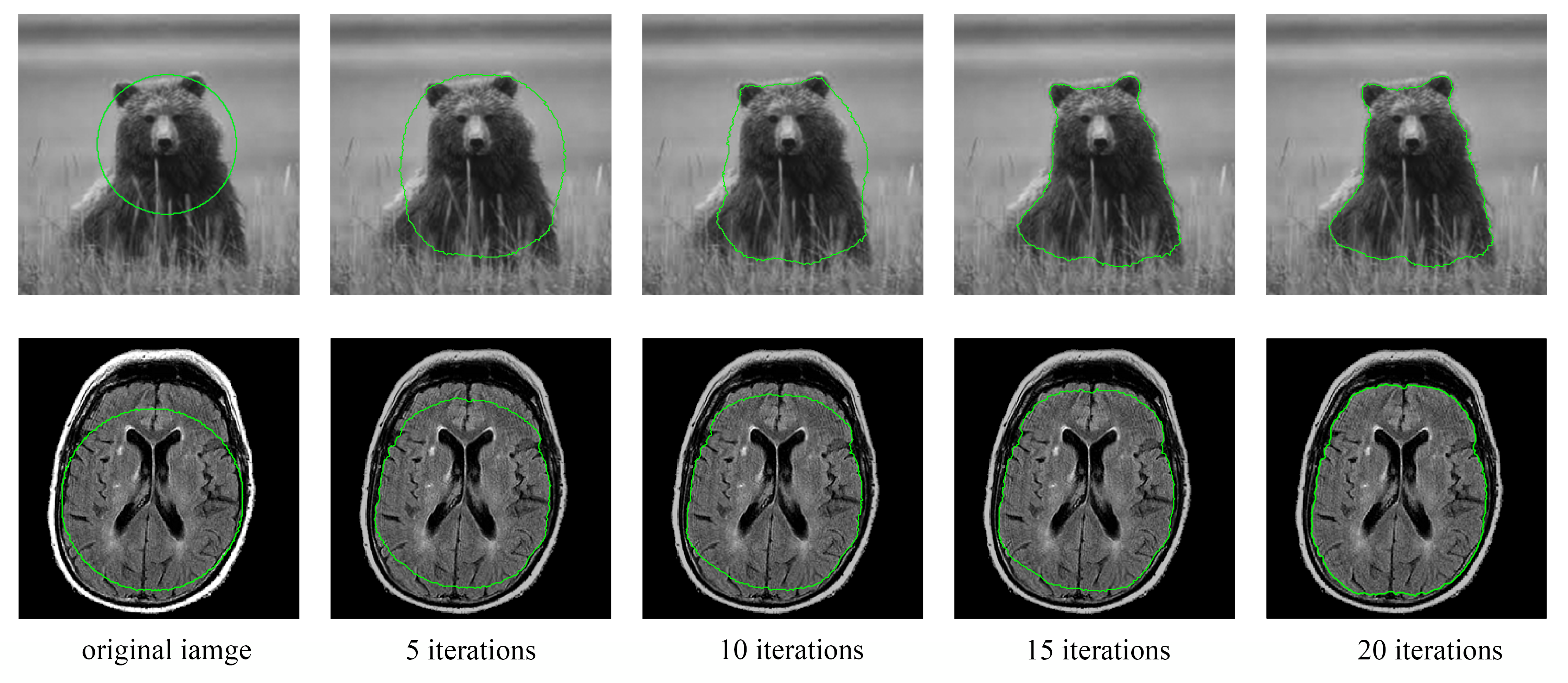}
    \end{center}
    \caption{The temporary segmentation results of the proposed model in different iterations. For each row, the green circle in the 1st column is the initial template, and the green curves in the 2nd to 5th columns are the segmentation results after 5, 10, 15, and 20 iterations.}
    \label{fig: seg process}
\end{figure}

\subsection{Impact of initial position}
Our model is a region-based segmentation method developed from the CV model. Thus, it shares common characteristics with algorithms of this category: it is sensitive to the initial position. In our model, the initial shape is fixed as a circle, and we need to set the center and radius to determine the initial position before the algorithm iteration begins. 

We conducted a detailed investigation into the impact of the initial position on the segmentation results, as illustrated in \cref{fig: init pos}. The image used for segmentation contains three circles, with the large white circle in the lower-left corner partially overlapping the smaller gray circle. In contrast, the third circle is far away in the upper-right corner. We try 11 different typical initial positions and obtain the corresponding results. 

As shown in (a), (b), and (c), when the initial circle is entirely within the target shape, we can achieve highly accurate segmentation results even in cases with overlapping regions. Meanwhile, (e), (f), and (g) indicate that when the initial circle contains the whole of the target region, the proposed model also generated reliable segmentation but may be affected by nearby objects. The inference of adjacent objects can be clearly observed in the comparison between (a2) and (e2), (b2) and (f2). In (i) and (j), we display the stable results under offsets of the initial circle. At the same time, we focus on the overlapping region in (k), whose segmented region (k2) does not grow continuously to fully encompass both circles, as is common in other algorithms. Instead, it is constrained by HBS and stops iterating, resulting in an outcome that does not conform to any edge. The segmented region of (d2) undergoes significant deformation and encompasses all three circles. However, the boundaries within the black background area still roughly resemble a circular shape, where such constraints of HBS are also evident. Finally, no meaningful segmentation is obtained if the initial position is set in an empty area as (h).
\begin{figure}
    \begin{center}
        \includegraphics[width=\textwidth]{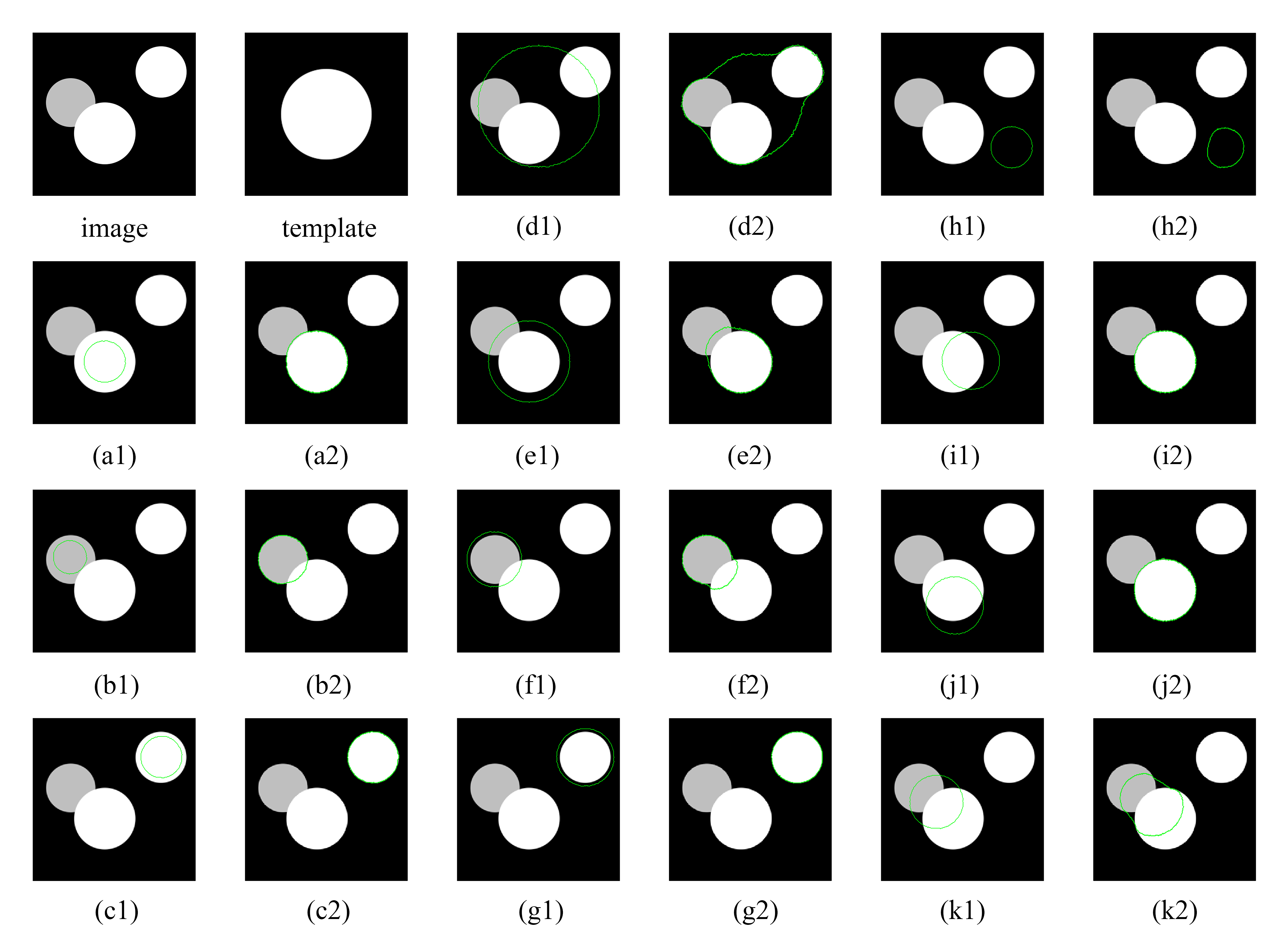}
    \end{center}
    \caption{The segmentation results under different initial positions. 11 sets of experimental results are presented, labeled from (a) to (k) in order. In each set of results, the green circle in the first image is the given initial position, while the green curve in the second image is the segmentation result.}
    \label{fig: init pos}
\end{figure}

In summary, the above experiments show that the proposed model can accurately locate the target regions based on the initial position. Additionally, we do not require precise initial position settings since the model is relatively insensitive and can accommodate reasonable offset and scaling of the initial circle. However, an inappropriate starting may yield an unexpected outcome. As mentioned earlier in \cref{section: exp: synthetic image}, the segmentation results can be influenced by adjacent objects, so avoiding other objects can help improve the segmentation performance.

\subsection{Segmentation on noised images}
Image degradation is ubiquitous in real-world computer vision applications, but segmentation of such low-quality images remains challenging. When algorithms falter in the presence of noise, their applicability diminishes substantially. Thus, advancing robust segmentation for corrupted images is essential. Shape priors could play a crucial role in this context by furnishing informative constraints during the processing pipeline. 

The following experiment demonstrates that our HBS segmentation allows for accurate analysis even when images are severely degraded by noise. 
In \cref{fig: noise seg}, we add Gaussian noise with mean $\mu$ and variance $\sigma$ to natural and synthetic images to simulate image degradation. The parameters of Gaussian noise and Signal to Noise Ratio(SNR) are marked under each subfigure. For the bear in 1st row, the clear original image and its segmentation result are displayed in \cref{fig: natural image} (a1) and (a5). For the split triangle in the 2nd row, the clear original image and segmentation result are (c1) and (c5) of \cref{fig: synthetic image}. 
% We can find that the segmentation task becomes more challenging as the noise intensity increases, but HBS allows us to obtain relatively stable segmentation results. In contrast, we have to admit that high noise levels inevitably lead to less smooth boundaries in the segmentation produced by our proposed algorithm. This effect is evident in (a2) and (a3), where noticeable jagged edges appear on both sides of the bear's head. Furthermore, the tolerance of our proposed algorithm to noise varies depending on the complexity of the image. 

It is evident that as the intensity of noise increases, the segmentation task becomes progressively more challenging, inevitably resulting in less smooth boundaries. However, leveraging HBS enables us to achieve relatively stable segmentation results. This effect is conspicuous in (a2) and (a3), where noticeable jagged edges emerge on both sides of the bear's head. 

Furthermore, the tolerance of our proposed algorithm to noise fluctuates depending on the complexity of the image.
For the split triangle, the segmentation performance remains good when $\sigma=0.08$ in (h). However, the algorithm can not handle $\sigma > 0.04$ effectively for bear.

\begin{figure}
    \begin{center}
        \includegraphics[width=\textwidth]{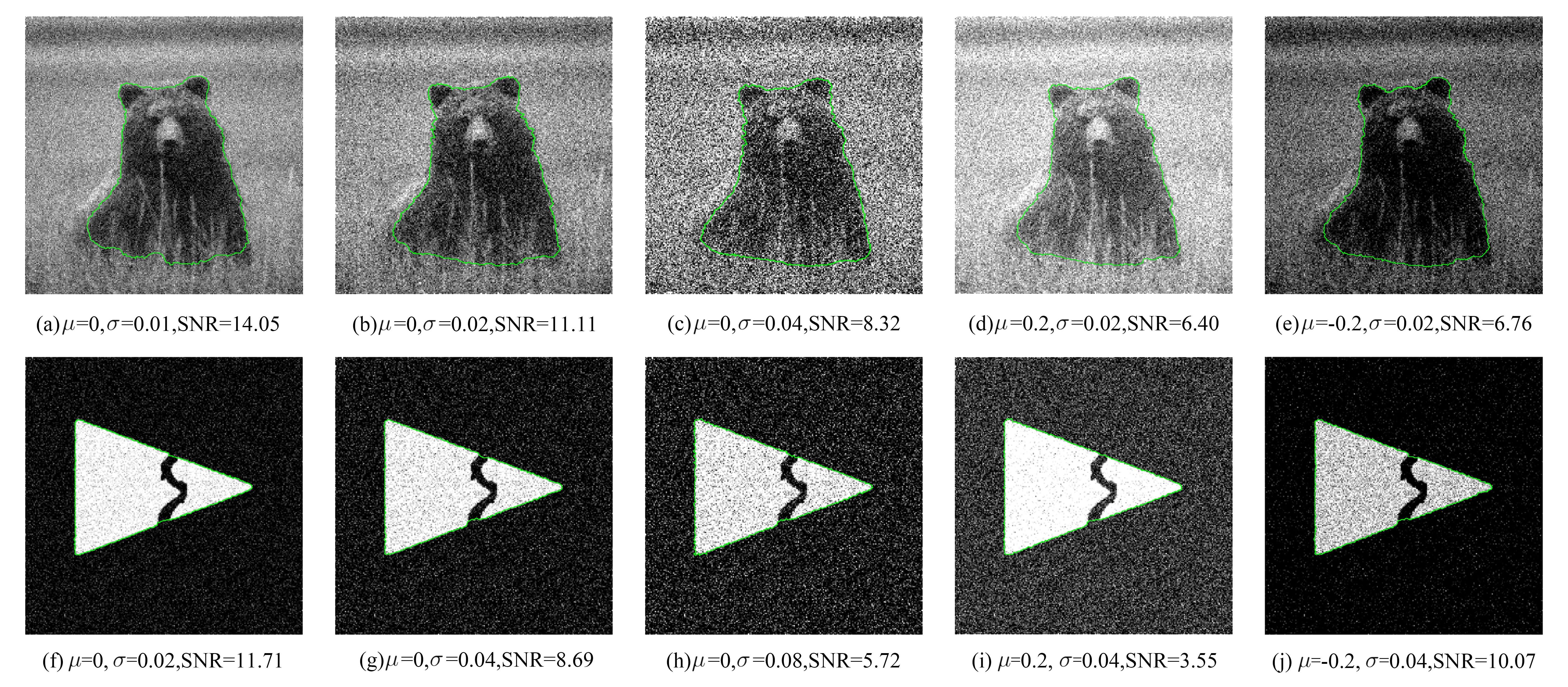}
    \end{center}
    \caption{Segmentation results on natural and synthetic images with Gaussian noise.}
    \label{fig: noise seg}
\end{figure}

\subsection{Enhancing object segmentation in multiple components scenes}
%In reality, objects are often composed of multiple parts in different shapes and have complicated boundaries. For example, a house may have a pentagonal outline, a triangular roof, rectangular doors, and circular windows. Avoiding interference from other parts and accurately locating the target region is a complex and challenging task. The HBS fully encodes the shape features and demonstrates excellent object recognition capabilities, performing well in classification tasks. Therefore, using HBS to provide additional shape prior information is likely a feasible approach to segment the desired area from images with multiple components. 
In reality, objects often have complex shapes with multiple components. For example, a house may feature a pentagonal outline, a triangular roof, rectangular doors, and circular windows. Accurately identifying and isolating the desired region amidst these intricate boundaries is challenging. The HBS effectively captures shape features and demonstrates strong object recognition capabilities, making it suitable for classification tasks. Therefore, using HBS to provide additional shape prior information offers a promising approach for segmenting specific areas in images with diverse components.

\begin{figure}
    \begin{center}
        \includegraphics[width=\textwidth]{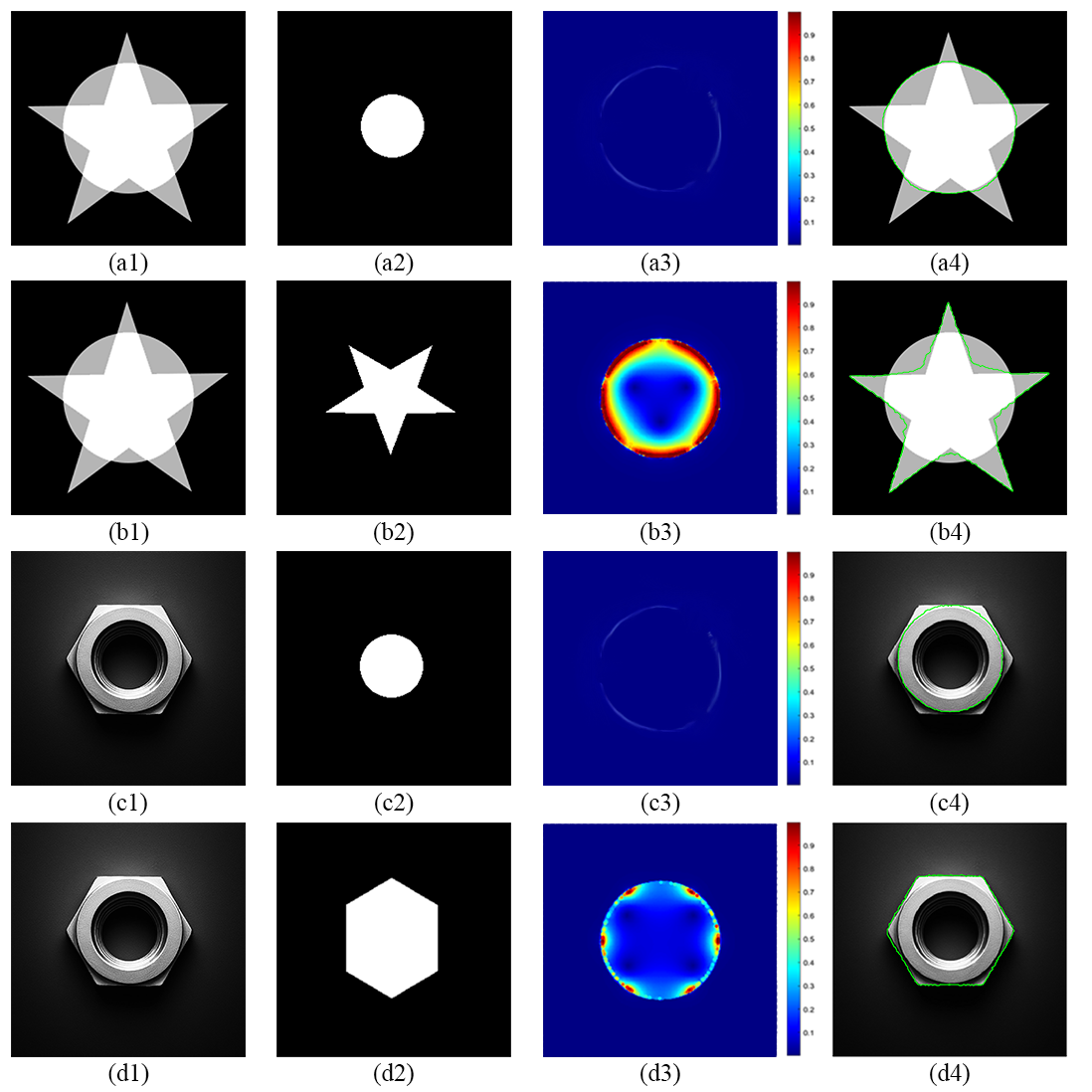}
    \end{center}
    \caption{The proposed model can obtain the corresponding segmentation results accurately according to the provided HBS when there are multiple composite shapes within the target region of the image. For each row, the sequence from left to right includes the original image, template shape, HBS of template, and segmentation result.}
    \label{fig: obj select}
\end{figure}

\cref{fig: obj select} is an attempt under situations described above and provides two sets of examples. The first set is tested on a synthetic image comprising a highly overlapping semi-transparent circle and pentagram. In this case, we aim to locate the circle and the pentagram, as shown in rows (a) and (b). The second set is tested on a hexagonal nut with an annular protrusion, and we attempt to segment the inner protrusion contour and the overall outline of the nut, as displayed in rows (c) and (d). The results (a4) to (d4) in the last column effectively demonstrate our model's ability, with the assistance of HBS, to accurately locate the target object in an image with multiple components according to the given requirements.

\subsection{Shape prior of object(s) with complicated boundary}
We have illustrated that shape prior can help segmentation, but obtaining shape prior remains a problem. In previous experiments, the templates fed into our model were always simple geometric shapes such as circles, triangles, etc. Although these simple shapes allow our model to achieve good results, they conceal the enormous potential of the HBS segmentation model.

While basic geometric shapes can be described using straightforward signatures like area, density, curvature, and circularity, their effectiveness diminishes as the complexity of the object increases. In such cases, traditional signatures struggle to capture the shape information fully, highlighting the capacity of the HBS to represent intricate shapes effectively. The HBS uses a simple and unified form to fully encode the characteristics of shapes, exhibiting a range of wonderful properties and enabling the reconstruction of the original shapes. Apart from extracting the shape boundaries, the HBS does not require any additional preprocessing to the image, which makes it extremely easy to find standard features from a large number of images of the same object class. 

In the medical field, physiological structures often exhibit specific shapes. Still, these shapes are complex and have a certain degree of variability, making their features hard to describe mathematically with clarity. \cref{fig: mean HBS and mean shape of brain} provides an example of computing shape prior of brains. (a1) to (a6) are the images of the brains with green lines as their boundary, (b1) to (b6) show their corresponding HBS, and (c1) to (c6) display the reconstructed brains from HBS. 
We can observe a small gap between the given brains in the 1st row and the reconstructed brains in the 3rd row, which is actually due to the resolution. However, it does not affect our segmentation since the reconstructed results are not used in our segmentation model at all. 
Besides, this gap will decrease when the resolution increases, as mentioned in \cite{linHarmonicBeltramiSignature2022}.
\begin{figure}
    \begin{center}
        \includegraphics[width=\textwidth]{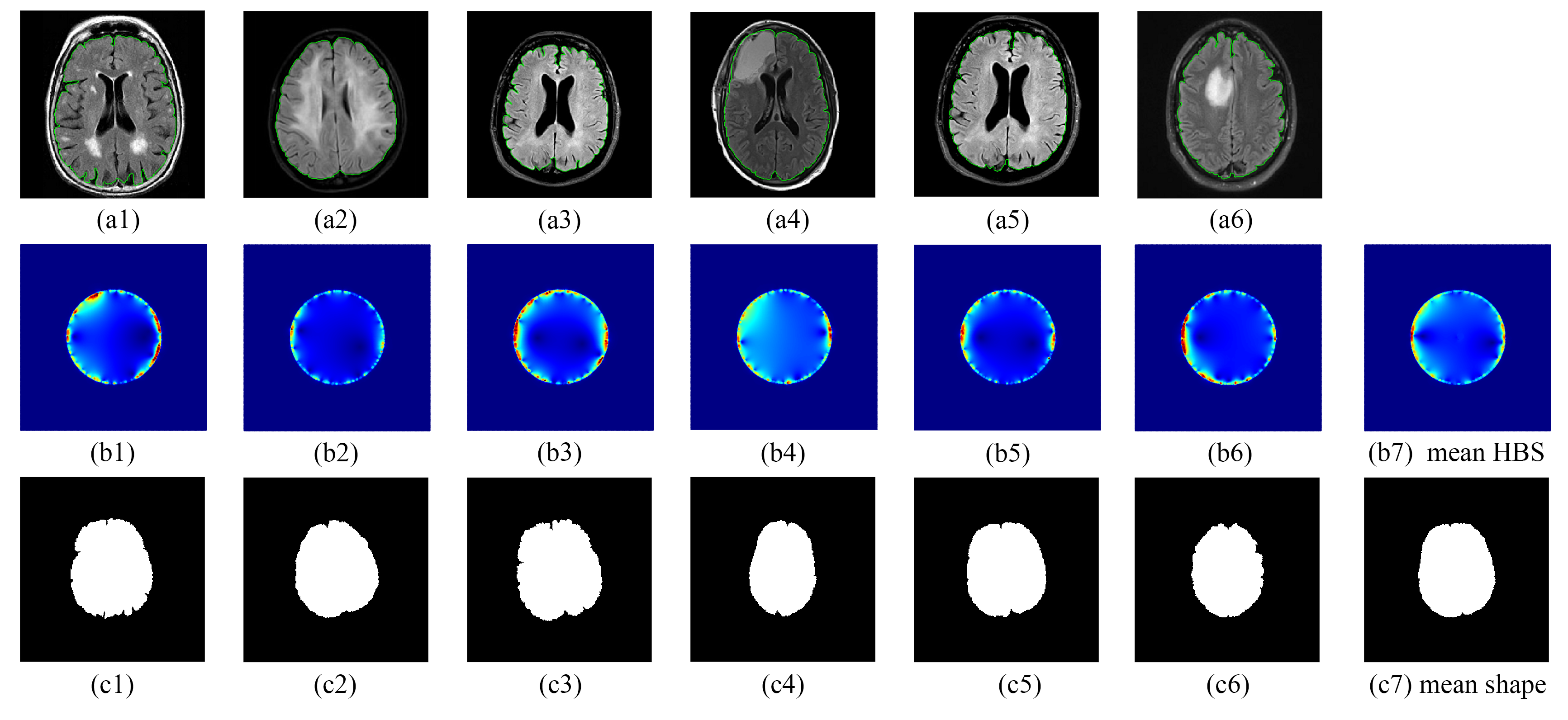}
    \end{center}
    \caption{HBS of brains and their mean HBS and mean shape.}
    \label{fig: mean HBS and mean shape of brain}
\end{figure}

Then we directly calculate the algebraic mean of the HBS of the above six brains, called the \textbf{mean HBS}, and also reconstruct a brain from the mean HBS, called the \textbf{mean shape}, which are respectively shown in the (b7) and (c7). Even though the relationship between HBS and the shape is challenging to explain, we can still observe the following features from (b1) to (b6):
\begin{enumerate}
    \item The majority of HBS values are small since brains appear roughly like circles.
    \item The HBS has larger values on the left side and smaller values on the right side, indicating a slightly tapering shape resembling the human brain.
    \item The HBS exhibits scattered high values along its boundary, corresponding to the sulcus in the brain's outline.
\end{enumerate}
The mean HBS also preserves these features, but the high values on the boundary become much less and are mainly concentrated at the left and right ends. This is reflected in the mean shape (c7), which has a smoother outline with distinct inward curvatures at the division between the left and right halves of the brain. Because of that, we believe that the mean HBS serves as a good shape prior.

It is noteworthy that utilizing the mean HBS as shape prior is not trivial, despite the fact that the mean HBS (b7) is exactly the HBS of the mean shape (c7). Achieving the mean shape directly poses challenges due to significant differences among shapes, such as variations in size, position, and angle. For example, prior to calculating the mean shape, brains (a1) to (a6) differ and are required to be realigned, resized, and oriented accordingly under traditional methods. Besides, all of the above operations rely on specific feature points, whose rationality and universality are also crucial factors to consider. In contrast, the HBS can circumvent these geometric discrepancies and capture potential shape features effectively. In summary, the mean HBS provides a generic and convenient method to extract shape prior with a solid theoretical background.

\subsection{Segmentation of complex image(s)}
With the help of the mean HBS, we can now apply our proposed segmentation model to complex images with suitable shape prior. 
Firstly, We take some new brains and the mean HBS in the last experiment as an example. 
For each row in \cref{fig: brain seg by mean HBS}, the 1st column is the given brain, the 2nd column is the segmentation result without HBS, the 3rd column is the segmentation result guided by the HBS of unit disk and the 4th column is the segmentation result guided by the mean HBS shown in (b7) of \cref{fig: mean HBS and mean shape of brain}.

\begin{figure}
    \begin{center}
        \includegraphics[width=15.5cm]{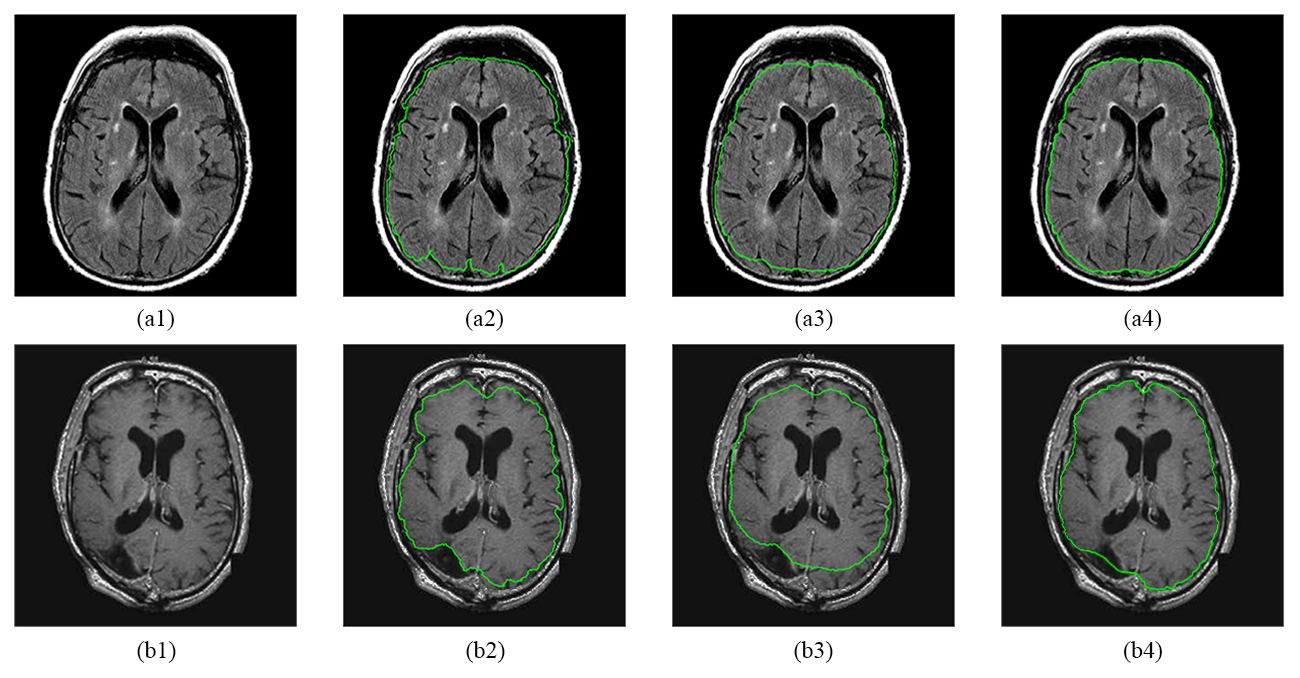}
    \end{center}
    \caption{Segmentation on brains guided by mean HBS. For each row, the sequence from left to right is the original image, the segmentation result without HBS, the segmentation result with HBS of the unit circle, and the segmentation result with mean HBS.}
    \label{fig: brain seg by mean HBS}
\end{figure}

In \cref{fig: brain seg by mean HBS}, the brain (a1) is quite similar to those used to compute mean HBS, (a1) to (a6) of \cref{fig: mean HBS and mean shape of brain}, leading that the result (a4) guided by mean HBS is almost perfect. The result (a3), guided by the HBS of the unit disk, loses some top and bottom regions, which makes it more circular. And the result (a2) without HBS has a very rough boundary. Due to some deep sulcus, it tends to discard regions with relatively discontinuous intensity, like the small part on the bottom of the brain.
The brain (b1) in the second row has a pathological region in the left bottom corner that deviates from our summarized standard features. Therefore, the segmentation result (b4) with mean HBS provides a relatively smooth and continuous contour while attempting to make up the missing region. This is precisely the desired outcome: areas affected by lesions or atrophy are still considered part of the brain. However, the other two results, (b2) and (b3), are not satisfactory enough and exhibit similar shortages with their corresponding ones in row (a).

We also conducted similar experiments on the brainstem and achieved favorable results. \cref{fig: mean HBS and mean shape of brainstem} displays the mean HBS (b6) extracted from (a1) to (a5), while \cref{fig: brainstem seg by mean HBS} illustrates the application of mean HBS for segmenting some new brainstem images (a1) and (b1).

\begin{figure}
    \begin{center}
        \includegraphics[width=\textwidth]{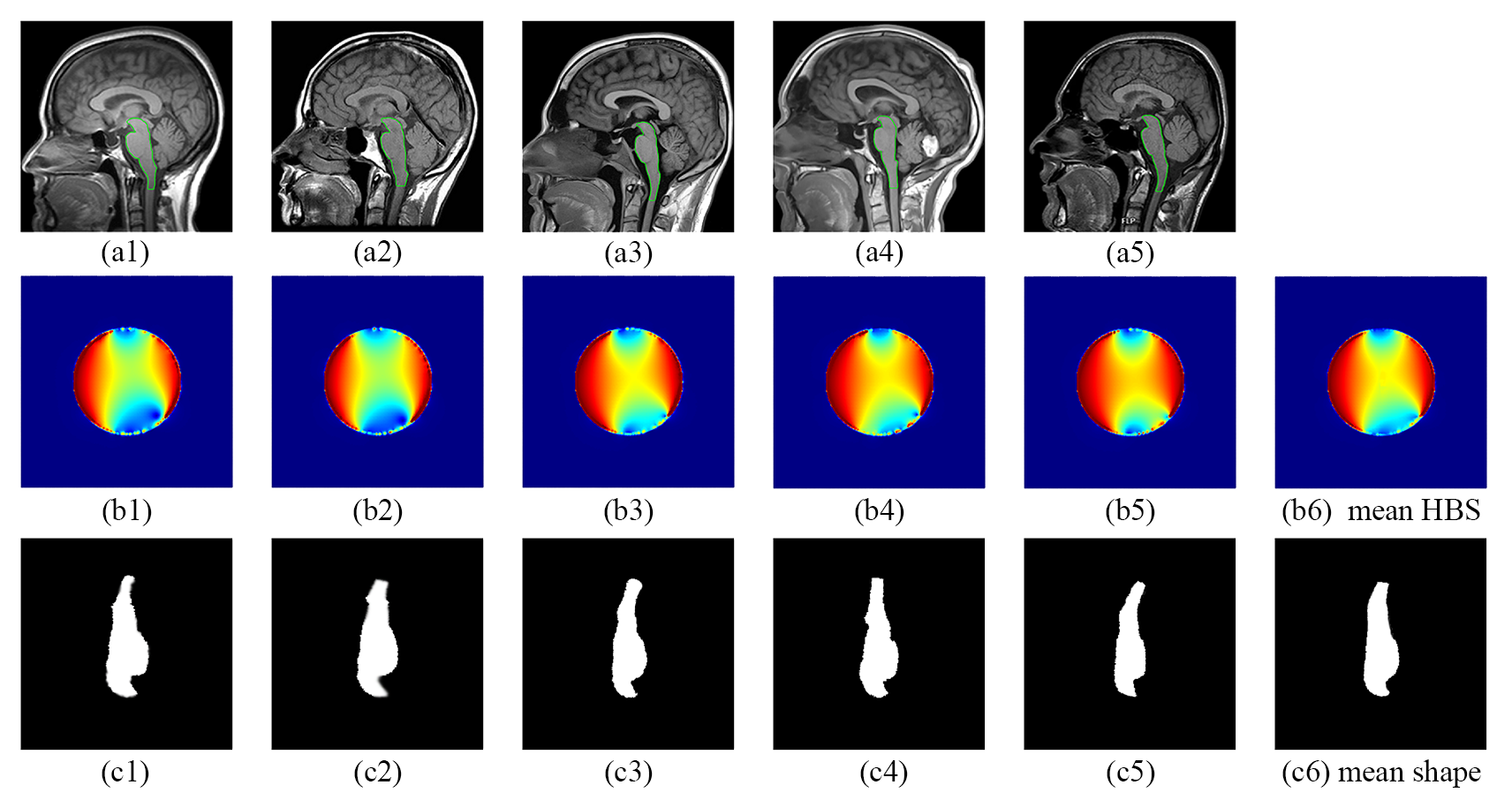}
    \end{center}
    \caption{HBS of brainstems and their mean HBS and mean shape.}
    \label{fig: mean HBS and mean shape of brainstem}
\end{figure}

\begin{figure}
    \begin{center}
        \includegraphics[width=8cm]{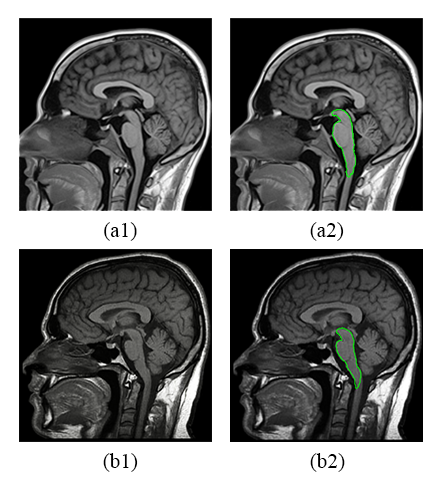}
    \end{center}
    \caption{Segmentation on the brainstem. For each row, the left is the original image, and the right is the result of our HBS segmentation model.}
    \label{fig: brainstem seg by mean HBS}
\end{figure}

\section{Conclusion}\label{section: conclusion}
This paper introduces a novel segmentation model leveraging HBS as shape prior. The key contributions and conclusions are summarized as follows:
\begin{enumerate}
    % \item The HBS is a powerful shape signature that fully captures the features of a 2D simply connected shape in a unified complex function defined on the unit disk with several desirable properties.
    \item We integrated the HBS into the quasi-conformal segmentation framework as a shape prior term, where the HBS guides the segmentation towards shapes similar to the provided prior.
    % \item The proposed model splits optimization into two iterative steps - finding a deforming map based on image differences and refining the map using the HBS prior. This provides an effective strategy to leverage the advantages of HBS.
    \item The HBS prior eliminates requirements for alignment, resizing, and orientation of shapes, and our HBS segmentation model naturally inherits such properties.
    \item Experiments on synthetic and natural images exhibited substantial improvements over the baseline model without HBS, validating the benefits of HBS in dealing with image degradation, occlusion, low contrast, and high noise.
    \item Segmentation results demonstrated accurate object localization in multi-part images by selecting an appropriate HBS prior, showing strong shape-discerning capabilities.
    \item Using the mean HBS computed from many related shapes provides an effective technique for obtaining a suitable shape prior, which is very useful for segmenting complicated images.
\end{enumerate}

In conclusion, the HBS prior segmentation model provides a novel, efficient, and flexible strategy to incorporate shape knowledge into segmentation algorithms. Substantial performance gains verify the advantages of the HBS representation. This paper illustrates the significant potential of HBS-guided techniques to advance a wide range of shape analysis tasks. Future work may explore HBS-based priors for 3D shapes and integrate the model into deep neural networks.

\section*{Acknowledgement} This work is supported by HKRGC GRF (Project ID: 14307622).

\bibliographystyle{plain}       % APS-like style for physics
\bibliography{paper}

\end{document}